%% file: root.tex
\pdfoutput=1

\documentclass[journal]{IEEEtran}
\input{preambles/preamble_math}
\input{preambles/preamble_acronym}
\input{preambles/preamble_general}

\begin{document}

\title{\LARGE \bf
TAMOLS: Terrain-Aware Motion Optimization for Legged Systems}
%
%
%

\author{
Fabian~Jenelten,~
Ruben~Grandia,~
Farbod~Farshidian,~
and~Marco~Hutter
\thanks{This research was partially supported by the Swiss National Science Foundation (SNSF) as part of project No.188596, the European Union’s Horizon 2020 research and innovation programme under grant agreement No.780883 and No. 101016970, and the Swiss National Science Foundation through the National Centre of Competence in Research Robotics (NCCR Robotics).}
\thanks{All authors are with the Robotic Systems Lab, ETH Zurich, 8092 Zurich, Switzerland. 
{\tt\footnotesize \{fabianje, rgrandia, farbodf, mahutter\}@ethz.ch}}
}

%
%




\maketitle

\begin{abstract}
\input{chapters/abstract}
\end{abstract}


\IEEEpeerreviewmaketitle

\input{chapters/main}
\appendices
\input{chapters/appendix}
\ifCLASSOPTIONcaptionsoff
  \newpage
\fi

\bibliographystyle{IEEEtran}
\bibliography{IEEEabrv,library}

\input{chapters/biography}

\end{document}

%% file: preambles/preamble_math.tex
\usepackage{amsmath}											
\usepackage{amssymb}                                			
\usepackage{upgreek}											
\usepackage{units}											
\usepackage{nicefrac}										
\usepackage{isomath}											
\usepackage{mathtools}
\usepackage{amsthm}                                             

\renewcommand{\vec}{\boldsymbol}								
\newcommand{\mat}{\boldsymbol}							 	
\newcommand{\Mx}[1]{\begin{bmatrix}#1\end{bmatrix}}			

\newtheorem{prop}{Proposition}
\newtheorem{assumption}{Assumption}
\newtheorem{definition}{Definition}

\usepackage{booktabs}																

%% file: preambles/preamble_acronym.tex
\usepackage[printonlyused]{acronym}

\acrodef{RSL}{Robotic Systems Lab}
\acrodef{COM}{center of mass}
\acrodef{SQP}{sequential quadratic problem}
\acrodef{VMC}{virtual model controller}
\acrodef{DOF}{degree of freedom}
\acrodefplural{DOF}[DOFs]{degrees of freedom}
\acrodef{ZMP}{zero moment point}
\acrodef{IMU}{inertial measurement unit}
\acrodef{COT}{cost of transport}
\acrodef{JPL}{Jet Propulsion Laboratory}
\acrodef{HAA}{hip adduction/abduction}
\acrodef{HFE}{hip flexion/extension}
\acrodef{KFE}{knee flexion/extension}
\acrodef{ZMP}{Zero-Moment Point}
\acrodef{QP}{quadratic program}
\acrodef{SQP}{sequential quadratic program}
\acrodef{WBC}{whole-body controller}
\acrodef{WBD}{whole-body dynamics}
\acrodef{HO}{hierarchical optimization}
\acrodef{COP}{Center of Pressure}
\acrodef{MPC}{Model Predictive Control}
\acrodef{HMM}{Hidden Markov Model}
\acrodef{PDF}{probability distribution function}
\acrodef{CDF}{cumulative distribution function}
\acrodef{CNN}{convolutional neural network}
\acrodef{TO}{trajectory optimization}
\acrodef{GM}{generalized momentum}
\acrodef{EOM}{equations of motion}
\acrodef{GIA}{gravito-inertia acceleration}
\acrodef{CD}{centroidal dynamics}
\acrodef{SRBD}{Single Rigid Body Dynamics}
\acrodef{CW}{contact wrench}
\acrodef{CWC}{centroidal or contact wrench cone}
\acrodef{GIW}{gravito-inertia wrench}
\acrodef{MMO}{modular motion optimization}
\acrodef{CMO}{combined motion optimization}
\acrodef{TAMOLS}{terrain-aware motion optimization for legged systems}
\acrodef{HAA}{hip abduction/adduction}
\acrodef{HFE}{hip flexion/extension}
\acrodef{LF}{left front}
\acrodef{LH}{left hind}
\acrodef{RF}{right front}
\acrodef{RH}{right hind}
\acrodef{NLP}{non-linear program}
\acrodef{GN}{Gauss-Newton}
\acrodef{DDM}{double-descriptor method}
\acrodef{GIAC}{gravito-inertia acceleration cone}

%% file: preambles/preamble_general.tex
\usepackage{xcolor}

\definecolor{grey}{gray}{0.4}

\usepackage{hyperref}
\hypersetup{
  colorlinks   = true, 
  urlcolor     = grey, 
  linkcolor    = blue, 
  citecolor   = red 
}

\usepackage{cite}

\usepackage[pdftex]{graphicx}
\ifCLASSINFOpdf
\else
\fi
\usepackage[caption=false,font=footnotesize]{subfig}

\usepackage{url}

\ifx\isReview\undefined
    \newcommand\highlight[1]{#1}
\else
    \newcommand\highlight[1]{\textcolor{red}{#1}}
\fi

%% file: chapters/abstract.tex
Terrain geometry is, in general, non-smooth, non-linear, non-convex, and, if perceived through a robot-centric visual unit, appears partially occluded and noisy. This work presents the complete control pipeline capable of handling the aforementioned problems in real-time. We formulate a trajectory optimization problem that jointly optimizes over the base pose and footholds, subject to a heightmap. To avoid converging into undesirable local optima, we deploy a graduated optimization technique. We embed a compact, contact-force free stability criterion that is compatible with the non-flat ground formulation. Direct collocation is used as transcription method, resulting in a non-linear optimization problem that can be solved online in less than ten milliseconds. To increase robustness in the presence of external disturbances, we close the tracking loop with a momentum observer. Our experiments demonstrate stair climbing, walking on stepping stones, and over gaps, utilizing various dynamic gaits.

%% file: chapters/main.tex
\section{Introduction}
\IEEEPARstart{L}{egged} locomotion has been studied and designed for the last couple of decades. Recent advances in both software and hardware have triggered the transition from experimental platforms used under laboratory conditions to (semi-) autonomous machines deployed in real-world scenarios, e.g., on industrial sites for inspection~\cite{Bellicoso2018} or in underground mines for exploration and mapping~\cite{Tranzatto2021}.
Yet, assumptions made in ``classical'' control approaches limit applications to flat or mildly rough ground or restrict locomotion to static stability. However, the true potential of legged locomotion is undeniably rooted in a combination of rough environments and dynamic agility. More recent control approaches have eliminated both the flat-ground and static-gait restrictions but struggle to match the computational overhead with the onboard compute budget. 

When it comes to rough terrain locomotion, we can isolate three major challenges: 
\begin{itemize}
\item Low computation time and generalizability over various terrains are contradicting requirements. Therefore, it is essential to strike a balance between simplicity and complexity to fully exploit perceptive feedback. 
\item When embedding the true terrain geometry, many local optima might appear, rendering a motion optimization problem sensitive to the initial guess. If the initializing motion is located on the ``wrong'' side of the cost function valley, the solution \highlight{might} converge into an undesirable local optimum. 
\item Sensors used to generate perceptive data are mounted onboard, and large parts of the field of view might appear occluded. These occluded regions can significantly shrink the feasible space or lead to unreliable foot placement.
\end{itemize}
In this work, we
address each of the three points individually, aiming to further
close the gap between blind and perceptive locomotion.
\begin{figure}
\centering
\includegraphics[width=1.0\columnwidth]{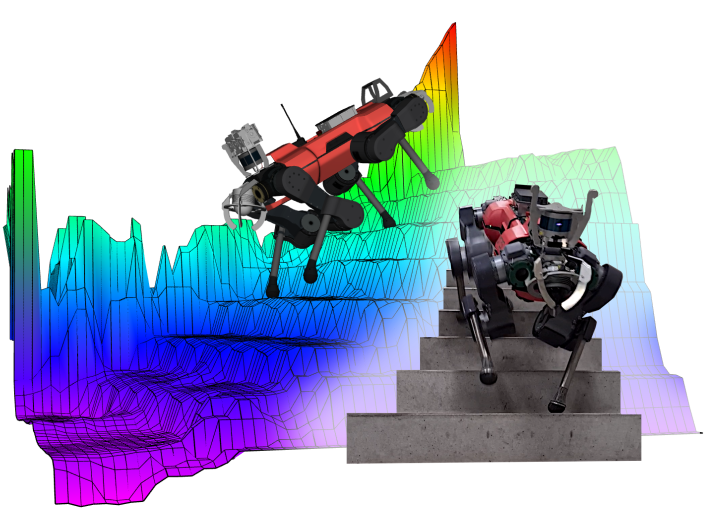}
\caption{Right: ANYmal~\cite{Hutter2016} is walking down a staircase using an ambling gait. Left: Corresponding gazebo model visualization with the elevation map.}
\label{fig:intro:title_image}
\end{figure}

\subsection{Contributions}
The main contributions of this work can be split into four independent parts. 1) First, we derive and analyze a contact-force free and fully differentiable dynamic stability metric. Compared to \ac{SRBD}, our model eliminates the contact forces and has thus a smaller problem size; as opposed to contact wrench models, we do not depend on algorithmic construction of the feasible set; and compared to \ac{ZMP} constraints, our model generalizes to uneven ground with vertical contact planes. 2) We present a \ac{TO} framework, capable of simultaneously optimizing over footholds and base pose, that can be solved at frequencies larger than \unit[100]{Hz}. Instead of handcrafting a reasonable initial guess, we rely on graduated optimization, a heuristic tool primarily used in computer vision for dealing with non-convex and non-smooth problems~\cite{Mobahi2015}. 3) We introduce a mapping pipeline and show how to extrapolate missing data in the map and how to remove artifacts and noise. 4) Finally, we improve the inverse dynamics based tracking loop by embedding a disturbance observer.

\subsection{Outline}
In the remainder of this section, we review different model-based \ac{TO} methods and categorize them according to complexity and decomposition of structure. 
Guided by this discussion, we introduce and justify design choices and model assumptions in section~\ref{s:design_choices}, forming the basis of our control pipeline. What follows in section~\ref{s:generalized_zmp_model} is the development of a dynamic stability criterion.
Mapping and map processing are carried out all onboard. Related details can be found in section~\ref{s:Perception}.
In the following section~\ref{s:optimization_method}, we present our approach to \acl{TAMOLS}\acused{TAMOLS}, hereinafter referred to as \ac{TAMOLS}. 
In section~\ref{s:wbc} we present changes to the widely used whole body control scheme.
The overall locomotion controller is evaluated on the quadrupedal robot ANYmal (see Fig.~\ref{fig:intro:title_image}) in section~\ref{s:results}, and concluded with section~\ref{s:conclusion}.

\subsection{Literature Review}
\subsubsection{Decomposition of Structure}
A first class of locomotion controllers, in this work termed \emph{\ac{MMO}}, decomposes motion optimization into several modules.
The number of modules and the module hierarchy may vary, but the separation of foothold and base pose planning is a necessary characteristics~\cite{BellicosoJenelten2017,Bellicoso2,Bledt1,Carlo1,Orsolino1,CROC,CCROC, Fankhauser2018,Magana2019, fabianje2020, kim2020}. The decoupling allows to simplify foothold selection in mapped environments, e.g., by iterating over a foot-score map and checking each grid cell for optimality. An often encountered concept, called foothold adaptation, limits the search space to the local vicinity of a nominal foothold, which is selected under blind conditions~\cite{Fankhauser2018,Magana2019, fabianje2020, kim2020}. Computation durations of such planners are low, but the lack of a true kinematic model can lead to infeasible base motions. 

As opposed to foothold adaptation, other work has focused on pattern generation. Similar approaches lay out a pattern of footholds along a height map, predicting touch-down locations over several strides. Scanning and map generation are either outsourced as offline tasks~\cite{Kolter1,Kalakrishnan1,Winkler2}, or the maps are generated onboard~\cite{Mastalli1,Griffin2019}. Optimization duration ranges from several tenths of a second to several seconds, rendering continuous replanning difficult on fast walking robots.

Locomotion planners characterized by a \emph{\ac{CMO}} can guarantee the existence of kinematically consistent trajectories. In combination with automatic gait discovery, this second class of locomotion controllers typically suffers from large computation times of several seconds to minutes~\cite{BiToQuad,zerodyn,Neunert3,Winkler1}.
When fixing the gait, these \ac{TO} problems have shown to be solvable in less than a second~\cite{Neunert2} or less than \unit[100]{ms}~\cite{Ruben1, regularized_pmpc2}. There exist also \ac{CMO} frameworks that optimize over the gait and terminate in several hundreds of milliseconds~\cite{Melon2021} or even less~\cite{Neunert1, Carius1}. Extending these formulations with perceptive information while keeping the optimization fast and reliable is still ongoing research.

\subsubsection{Perceptive Locomotion}
Most work in the field of perceptive locomotion is inspired by the early development of Kolter et al.~\cite{Kolter1}. Footsteps are found using a greedy search over a foot-cost map along an approximate body path. The latter is subsequently refined offline by planning a statically stable \ac{COM} trajectory. The controller's performance was evaluated with \emph{LittleDog} on four different terrains of varying difficulty. Kalakrishnan et al.~\cite{Kalakrishnan1} proposed several extensions to that approach; the most notable one is online \ac{COM} planning subject to \ac{ZMP} stability. Both methods rely on accurately pre-scanned terrains and pre-processed height maps. 

Throughout the following years, related \ac{MMO} approaches have been pushed towards full onboard motion planning. Mastalli et al.~\cite{Mastalli1} achieved a re-planning frequency of \unit[0.5]{Hz} for the foothold computation. The closed-loop pipeline was evaluated with \emph{HyQ}, demonstrating crawling over stacked pallets and stepping stones. The key to fast convergence is to restrict the search space to an approximate body path plan. Griffin et al.~\cite{Griffin2019} suggested a convex decomposition of the environment and introduced an $A^*$-footstep planer. It also accounts for yaw orientation of the feet and converges in less than $\unit[2]{s}$. Performance was verified with the biped robots \emph{Atlas} and \emph{Valkyrie} on a set of different terrains. 

Following the \ac{MMO} structure, Fankhauser et al.~\cite{Fankhauser2018} introduced a terrain-aware locomotion pipeline, where footholds are selected from a binary foothold-score map. The search space is limited to the local surroundings of a priory computed nominal foothold. Using \emph{ANYmal}, the authors showed statically stable rough terrain locomotion. 

Magana et al.~\cite{Magana2019} presented a self-supervised foothold classifier. The convolutional neural network learns a correction step based on a nominal foothold. The approach was verified on HyQ, demonstrating a trot over gaps and pallets. This approach was further extended by Jenelten et al.~\cite{fabianje2020} using a batch search, showing dynamic stair-climbing with ANYmal.

Gangapurwala et al.~\cite{RLOC} employed a learning-based policy to plan perceptive footholds, tracked within a model-based \ac{MMO} environment. Robustness of the tracking controller has been improved by learning corrective joint torques. The combined framework has been tested with ANYmal on slopes and bricks, utilizing a trotting and crawling gait.

The more complex \ac{CMO} formulations are on a transition to online-perceptive locomotion. Mastalli et al.~\cite{Mastalli3} presented a method where the motion is subject to \ac{ZMP}-stability, and the footholds are constrained to a terrain cost map. The stochastic search leads to optimization durations of several minutes, and hence, the \ac{TO} has to be carried out offline.

First successful results in generalizing kino-dynamic \ac{MPC} has been shown by Grandia et al.~\cite{Ruben3}. ANYmal was demonstrated to trot over dense stepping stones with the help of a pre-mapped and pre-segmented terrain.

In the work of Melon et al.~\cite{towr_learning}, the \ac{CMO} pipeline proposed by Winkler et al.~\cite{Winkler1} has been improved with a learning-based initializer. In the experiments performed with ANYmal, the environment was loaded from a virtual model. Most recently~\cite{Melon2021}, this approach was shown in combination with online mapping, enabling ANYmal to walk over \unit[20]{cm} high steps. Despite promising results, the replanning frequency was reported as not larger than $\unit[3.3]{Hz}$ in the most favorable case.

\subsubsection{Dynamic Model}
\begin{figure*}
\centering
\includegraphics[width=1.0\textwidth]{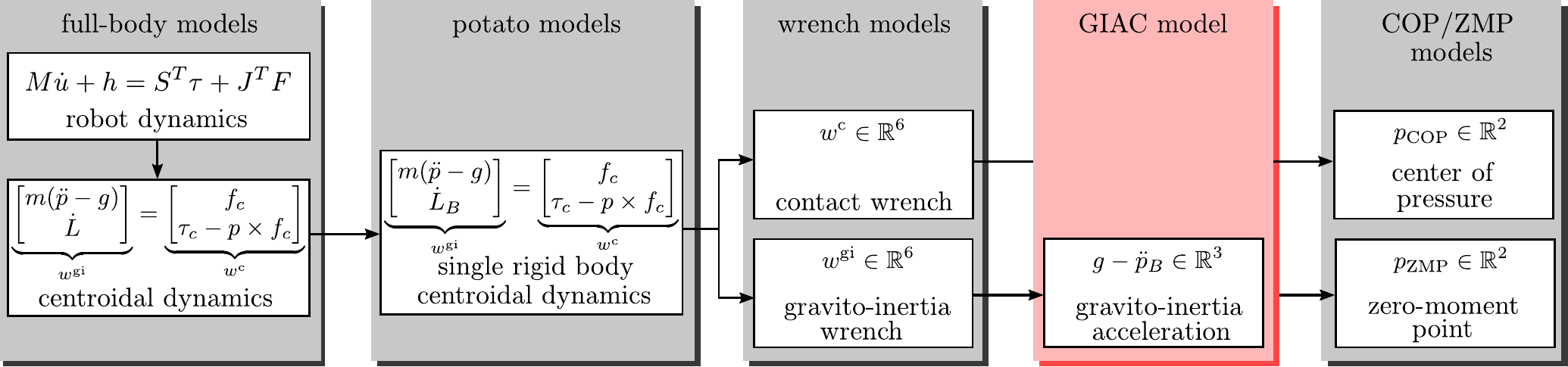}
\caption{Illustration of dynamic models following the order of regressive complexity (or dimensionality) from left to right. The whole-body dynamics represent the robot as a fully articulated structure. The governing equations of motion (EOM) are obtained by Lagrangian mechanics. The first six rows are called the centroidal dynamics (CD) and coincide with the Newton–Euler \acs{EOM}. It relates variations of the centroidal momentum with the sum of contact forces $\vec f_c = \sum \vec f_i$ and momenta $\vec \tau_c = \sum \vec p_i \times \vec f_i$. The \acs{SRBD} account for only one rigid body, and the angular centroidal momentum $\vec L$ simplifies to the angular momentum of the base $\vec L_B$. Gravito-inertia/contact wrench models consider only the left/right-hand side of the \acs{CD}. The \acs{GIAC} model, which will be presented in this work, only depends on the gravito-inertia acceleration and footholds. \acs{ZMP}/\acs{COP} models are the projection of the associated wrench onto the ground or some virtual plane.}
\label{fig:models:dynamic_models_overview}
\end{figure*}
The formulation of a \ac{TO} problem requires a physical model at hand. We have categorized these models based on their complexity as in the following and as summarized in Fig.~\ref{fig:models:dynamic_models_overview}.

\ac{TO} has been used in combination with the \emph{\ac{WBD}} of legged robots. These methods can directly constrain joint torques. Due to excessive model complexity, optimization times usually peak at several minutes~\cite{zerodyn, Neunert3} or seconds~\cite{BiToQuad, Neunert2}, but can be as low as several tenths of a millisecond\cite{Neunert1, Carius1}.

The \emph{\ac{CD}} are the projection of the \ac{WBD} to the \ac{COM}~\cite{Orin2}. It has been used in \ac{TO} methods for bipedal~\cite{centroidal_mpc1, centroidal_mpc2} and quadrupedal systems~\cite{Farbod, sleiman1}. 

The joint space states can be eliminated by assuming that the robot's inertia remains constant. These models are commonly referred to as \emph{\ac{SRBD}}.
Related methods can integrate kinematic constraints formulated in joint~\cite{Ruben1} or task space~\cite{Winkler1,towr_learning, Melon2021}. A popular approach is to linearize the dynamics around a nominal operation point, e.g. the measured state~\cite{lin_centroidal_dyn1}, steady state~\cite{Carlo1, regularized_pmpc1, regularized_pmpc2} or desired state~\cite{kim2019}.

The \ac{CD} or \ac{SRBD} can be further decomposed into two parts, the \emph{\ac{CW}} and the \emph{\ac{GIW}}~\cite{CaronZMP}, separating contact forces from spatial acceleration. The \ac{GIW} is typically used for planning where it is subject to \ac{CWC} feasibility. Due to the absence of contact forces, this stability criterion scales better with the number of contacts~\cite{cwc1}. To the authors' best knowledge, methods of this category have only been used in combination with prefixed footholds~\cite{cwc2, Caron2015, Orsolino1, CROC, CCROC}, for which the \ac{CWC} can be pre-computed as a conic polyhedron.

Projecting the \ac{GIW} or the \ac{CW} onto the ground~\cite{zmp} or a virtual plane~\cite{CaronZMP} leads to the definition of the so-called \emph{\acf{ZMP}} or \emph{\acf{COP}}, respectively. Both have different meanings but are equivalent on a physical basis. Under the same transformation, the \ac{CWC} projects to a support polygon, defined as the convex hull of (virtual) contact positions. These reduced models represent some kind of inverted pendulum with a single virtual contact, from which the most popular one is the linear inverted pendulum mode ~\cite{Kalakrishnan1, BellicosoJenelten2017, Bellicoso2}. If we further assume zero acceleration, the \ac{ZMP}/\ac{COP} simplifies to the projected \ac{COM}~\cite{Fankhauser2018, Kolter1}. The associated stability criterion is of pure static nature. 

\section{Overview} \label{s:design_choices}
\subsection{Design Choices And Model Assumptions}
\subsubsection{Decomposition of Structure}
In rough terrains, the optimal motion is often found at kinematic limits. The absence of kinematic constraints in the foothold selection step renders \ac{MMO} susceptible to leg over-extension. Footholds might be located too far from each other, potentially reducing the feasible space of the base pose trajectory to the empty set. On the other hand, if kinematic constraints are not taken into account, the limbs may be exposed to singular configurations where tracking performance degrades~\cite{fabianje2020}. 
When combining base pose and footholds into one single \ac{TO} problem, we can directly constrain limb extension and thereby enforce a feasible task space configuration without the need for an exact kinematic model. In this work, the focus is put on such a \ac{CMO} method.

The gait pattern defines lift-off and touch-down timings of the feet over a periodic time interval.
As a second design choice, our approach uses a fixed gait pattern. This is not a necessary assumption, but it reduces the problem size and thus the expected optimization time.

\subsubsection{Dynamic Stability Criterion}
The literature review suggests a correlation between model complexity and optimization duration. In favor of the latter, we prefer the model to be located on the far right side of Fig.~\ref{fig:models:dynamic_models_overview}.

In the case of ANYmal, the four limbs make up for half of the robot's total weight. Two third of the limb mass is accumulated in the \ac{HAA} and \ac{HFE} joints, forming an approximate revolute joint. Its mass is located close to the torso, which makes the \ac{SRBD} a plausible model for our robot.

As a next simplification, we might consider only the total wrench generated by the contact forces and the \ac{CWC} that it belongs to. Mathematically, such a cone is formulated as the Minkowski sum of translated friction cones, for which there exists no closed-form solution. Numerical tools exist, such as the \ac{DDM}~\cite{double_description}. Unfortunately, the algorithm finds an expression that is non-differentiable w.r.t. footholds, and thereby, wrench models can not be used in our case. A brief introduction into \ac{CWC} stability is provided in the appendix~\ref{a:comparison_cwc}.

The simpler \ac{ZMP}/\ac{COP} models assume flat ground and are by design not representative in rough environments.

We conclude that the least complex design that suffices the requirements for perceptive locomotion is based on \ac{SRBD}. In the subsequent section, we will introduce model assumptions that allow us to eliminate contact forces while maintaining differentiability w.r.t. all unknowns.

\subsection{Nomenclature}\label{ss:models:nomencalture}
If not stated differently, vectors are expressed in \emph{world frame} $W$, defined by the orthogonal unit vectors $\{\vec e_x, \vec e_y,\vec e_z\}$ and the origin $\mathcal{O}_W$. The robot-centric \emph{base frame} $B$ is attached to the geometric center of the base. We use $\vec x^B$ notation to refer to a vector given in local coordinates $B$.

We denote a foothold and a contact force of leg $i$ by $\vec p_i$ and $\vec f_i$, respectively, with \highlight{$i = 1,\ldots, N$} and $N$ the number of grounded legs. A difference between two footholds $i$ and $j$ is written as $\vec p_{ij} = \vec p_j - \vec p_i$. In case of a quadrupedal robot, the four legs are labeled \highlight{\ac{LF} ($i=1$), \ac{RF} ($i=2$), \ac{LH} ($i=3$), and \ac{RH} ($i=4$)}. 

For the base pose we use the expression $\vec \Pi_B = \Mx{\vec p_B & \vec \phi_B}$ where $\vec p_B = \Mx{x & y & z}^T$ is the base position in world frame and $\vec \phi_B = \Mx{\psi &\theta& \varphi}^T$ are the ZYX-Euler angles of the base w.r.t. world. We write $\mat R_B = \mat R_B(\vec \phi_B)$ to indicate a rotation matrix that rotates a vector expressed in base coordinates to world frame. For reducing index complicity, we assume that the base center is identical to the \ac{COM} of the base.

We compute linear and angular momentum of the base as $\vec P_B = m \vec{\dot p}_B$ and $\vec L_B = \mat R_B (\mat I_B \vec \omega_B^B)$.
The vector $\vec \omega_B^B = \vec\omega_B^B(\vec \phi_B, \vec {\dot \phi}_B)$ is the angular velocity of the base around the base frame, $m$ is the total mass, and $\mat I_B$ the inertia of the robot in base frame computed at nominal configuration.

The gravity vector is written as $\vec g = \Mx{0 & 0 & -g}^T$ with $g = \unit[9.81]{m/s^2}$. The symbol $\mu$ is used to denote a friction coefficient. We define a wrench $\vec w \coloneqq \{\vec f, \, \vec \tau\}$ as the tuple consisting of a force and moment.
The \emph{\ac{GIA}} vector is defined as $\vec a_B \coloneqq \vec g - \vec {\ddot p}_B$. Finally, we use the determinant notation $\det(\vec a, \vec b, \vec c)$ to express the scalar triple product $\vec a \cdot (\vec b \times \vec c)$.

\subsection{Control Structure}
The control architecture used in this work is outlined in Fig.~\ref{fig:intro:overview_tamols}.
\begin{figure}
\centering
\includegraphics[width=1.0\columnwidth]{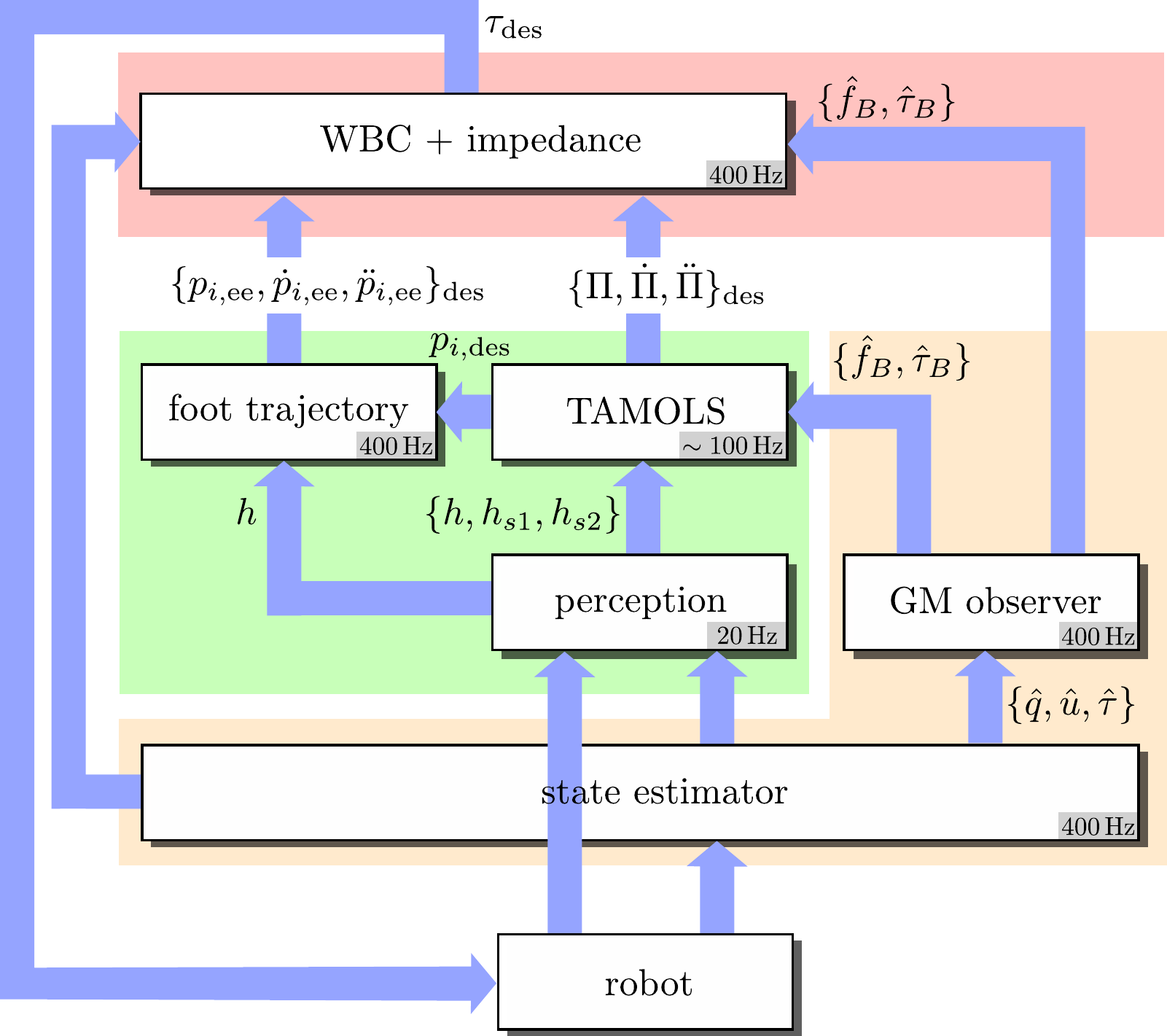}
\caption{Overview of the control structure, consisting of trajectory generation (green), tracking (red), and state/disturbance observer (orange).}
\label{fig:intro:overview_tamols}
\end{figure}
Generalized positions $\hat{\vec q}$, generalized velocities $\hat{\vec u}$ and torques $\hat{\vec \tau}$ are estimated/measured at \unit[400]{Hz}. The point clouds gathered from two onboard LiDARs are processed to different elevation maps $\{h,\,h_{s1},\,h_{s2}\}$ at a rate of \unit[20]{Hz}. The three elevation layers are tightly embedded in the \ac{TO} problem and the swing trajectory generation. These two modules output desired reference signals for the base $\{{\vec \Pi}_B,\, \dot {\vec \Pi}_B,\, \ddot {\vec \Pi}_B\}_\text{des}$ and swinging feet $\{{\vec p}_{i,\text{ee}}, \, \dot {\vec p}_{i,\text{ee}}, \, \ddot {\vec p}_{i,\text{ee}}\}_\text{des}$, which are tracked by a \ac{WBC} in task space at \unit[400]{Hz}. External disturbances of the base $\{\hat {\vec f}_B,\,\hat {\vec \tau}_B\}$, estimated by a \ac{GM} observer, are compensated in the tracking controller and in the motion optimizer.

\section{The \acs{GIAC} model} \label{s:generalized_zmp_model}
We first recapitulate the \ac{SRBD} in~\ref{ss:giac_background}, and show how to eliminate the contact forces in~\ref{ss:giac_derivation}, \highlight{given the assumptions to be introduced in~\ref{ss:model_assumptions}}. The resulting model depends on the contact state and is summarized in~\ref{ss:giac_summary}. The remaining subsection~\ref{ss:giac_properties} discusses its inherent geometric and physical properties.

\subsection{Background} \label{ss:giac_background}
A legged robot can be modeled as a floating body that balances contact forces through the rate of change of linear and angular momentum,
\begin{equation} \label{eq:models_cd}
    \underbrace{m \Mx{ 
    \vec {\ddot p}_B - \vec g \\ 
    m\vec p_B \times (\vec {\ddot p} - \vec g)
    } + 
    \Mx{ \vec 0 \\ \vec {\dot L} }}_{\coloneqq -\vec w^\text{gi}} =
    \underbrace{\sum_{i=1}^N 
    \Mx{
    \vec f_i \\ \vec p_i \times \vec f_i
    }}_{\coloneqq \vec w^\text{c}}.
\end{equation}
The left/right-hand side is called the gravito-inertia/contact wrench, while the combination of both is known as the \ac{CD}. If the limb masses are negligible, then the angular momentum $\vec {L}$ simplifies to the angular momentum of the base $\vec {L}_B$, and the resulting equation is called \ac{SRBD}.

We are not only interested in motions that are consistent with physics, but that also support non-sliding contacts. If friction is modeled using Coulomb's friction law, such a requirement can be formulated as
\begin{equation} \label{eq:models_cd_friction_cone}
    \vec f_i \in \mathcal{F}_i, \quad \forall i = \highlight{1, \ldots, N}.
\end{equation}
The set $\mathcal{F}_i$ is called the \emph{friction cone} and is defined by the plane normal $\vec n_i$ and friction coefficient $\mu_i$ at contact $i$
\begin{equation}
    \mathcal{F}_i = \Big\{\vec f_i \mid \mu_i \vec n_i^T \cdot \vec f_i \geq ||(\mathbb{I}_{3\times 3} - \vec n_i \vec n_i^T) \vec f_i|| \Big\}.
\end{equation}

\subsection{\highlight{Model Assumptions}} \label{ss:model_assumptions}
We start the mathematical derivation by introducing the following set of assumptions:
\begin{assumption}[Single body]\label{assumtopns:models:1}
The limbs have zero mass.
\end{assumption}
\begin{assumption}\label{assumtopns:models:2}
The  rate of  change  of  the  angular  momentum  has a negligible effect on the contact forces.
\end{assumption}
\begin{assumption}\label{assumtopns:models:3}
Contact forces can only push on the ground.
\end{assumption}
\begin{assumption}[Coulomb friction]\label{assumtopns:models:4}
Friction coefficients are constant \highlight{over time and identical for all legs}.
\end{assumption}
\begin{assumption}\label{assumtopns:models:5}
Contacts are established on horizontal planes.
\end{assumption}
\begin{assumption}[Over-hanging torso]\label{assumtopns:models:6}
The base position is located above all the grounded feet, i.e., $\vec e_z^T\cdot \vec p_B > e_z^T\cdot \vec p_i~\forall$ grounded $i$.
\end{assumption}

\subsection{Derivation} \label{ss:giac_derivation}
By summing up the contact forces in~\eqref{eq:models_cd_friction_cone}, we obtain a condition on the \ac{GIA} vector
\begin{equation} \label{eq:models_zmp_friction_cone}
    \sum_{i=1}^N \vec f_i = m(\vec{\ddot p}_B - \vec g) \implies \vec{\ddot p}_B - \vec g \in \mathcal{F}.
\end{equation}
The set $\mathcal{F}$ is obtained as the Minkowski sum of local friction cones $\{\mathcal{F}_1,\ldots, \mathcal{F}_N\}$. Assumptions~\ref{assumtopns:models:3}, \ref{assumtopns:models:4}, and~\ref{assumtopns:models:5}  lead to a relaxed scenario, in which an analytical expression of the friction cone can be found as
\begin{equation}
\mu \vec e_z^T \cdot (\vec g - \vec{\ddot p}_B ) \geq ||(\mathbb{I}_{3\times 3} - \vec e_z \vec e_z^T) (\vec g - \vec{\ddot p}_B )||. \label{eq:models:simplified_no_slip_condition}
\end{equation}
Assumption~\ref{assumtopns:models:2} guarantees that friction constraints on the linear momentum are sufficient for non-sliding contacts.

Without loss of generality, we exploit assumption~\ref{assumtopns:models:1}, i.e.,  $\vec {\dot L} = \vec {\dot L}_B$, and rewrite the system dynamics as
\begin{subequations}
\begin{gather}
    m \Mx{ 
    \vec {\ddot p}_B - \vec g \\ 
    \vec 0
    } + 
    \Mx{ \vec 0 \\ \vec {\dot L}_B } =
    \sum_{i=1}^N 
    \Mx{
    \vec f_i \\ (\vec p_i - \vec p_B) \times \vec f_i
    } \label{a:sbcd} \\
    \vec f_i \in \mathcal{F}_i, \quad \forall i \in \{1, \ldots, N\} \label{a:no_slip_condition}.
\end{gather}
\end{subequations}

In the following, we show how to eliminate the contact forces from~\eqref{a:sbcd} and~\eqref{a:no_slip_condition} by using assumptions~\ref{assumtopns:models:5} and~\ref{assumtopns:models:6}.

\subsubsection{Three or more grounded Legs}
If the number of contact locations is $N> 1$, then we can solve the first row of~\eqref{a:sbcd} for $\vec f_1$ and substitute into the second row, 
\begin{equation} \label{eq:a:cd_combined}
m (\vec p_1 - \vec p_B) \times (\vec {\ddot p}_B - \vec g) +
\sum_{i\neq 1} (\vec p_i - \vec p_1) \times \vec f_i  = \vec {\dot L}_B. 
\end{equation}
Next, we multiply \eqref{eq:a:cd_combined} from the left with $\vec p_{12} = \vec p_2 - \vec p_1$
\begin{multline}
m \vec p_{12}^T \cdot \big[(\vec p_1 - \vec p_B)^T \times (\vec {\ddot p}_B - \vec g) \big]+ \\ 
\vec p_{12}^T \cdot \sum_{i\neq 1,2} (\vec p_i - \vec p_1) \times \vec f_i = \vec p_{12}^T \cdot \vec {\dot L}_B.
\end{multline}
Using the determinant notation and the circular shift property of the scalar triple product, the equation can be simplified to
\begin{multline}
\det(\vec p_{12}, \vec p_B-\vec p_1, \vec g-\vec {\ddot p}_B) - \vec p_{12}^T \cdot \frac{\vec {\dot L}_B}{m} = \\
-\frac{1}{m}\sum_{i\neq 1,2} \vec f_i^T \cdot \big[\vec p_{12} \times (\vec p_i -\vec p_1)\big].
\end{multline}
We seek for a condition under which the left-hand side is smaller than zero, i.e., for which
\begin{equation} \label{eq:a:zmp_condition}
\vec f_i^T \cdot \big[\vec p_{12} \times (\vec p_i -\vec p_1)\big] \geq 0, \quad \forall i \neq 1,2.
\end{equation}
The three contact locations $\{\vec p_1,\, \vec p_2,\,\vec p_i\}$ span a local terrain plane $\vec \Pi_{12i}$, whose normal is parallel to $\vec p_{12} \times (\vec p_i -\vec p_1)$.  In the most restrictive case, there is no friction, $\mu_i = 0$, implying that~\eqref{eq:a:zmp_condition} needs to be satisfied for $\vec f_i = \vec n_i$. We write this condition as $\angle\{\vec\Pi_{12i}, \vec n_i\} \geq 0$. It is satisfied if the ground can be represented by a single plane (with arbitrary inclination) or segmented into multiple planes perpendicular to gravity.

With these considerations at hand we can conclude that:
\begin{prop} \label{a:dyn_prop1}
There exists a set of contact forces $\{\vec f_i\}_{\highlight{i=1,\ldots, N}}$ satisfying \eqref{a:sbcd} and \eqref{a:no_slip_condition}, if
\begin{itemize}
    \item the \ac{GIA} vector is bounded by
    \begin{equation}
        m\det(\vec p_{ij}, \vec p_B-\vec p_i, \vec g-\vec {\ddot p}_B) \leq \vec p_{ij}^T \cdot \vec {\dot L}_B,
    \end{equation}
    \item the \ac{GIA} vector is bounded by the friction cone~\eqref{eq:models:simplified_no_slip_condition},
    \item \highlight{assumption~\ref{assumtopns:models:2} is satisfied}, and
    \item $\angle\{\vec\Pi_{jki}, \vec n_i\}\geq 0,~j\neq k \neq i$.
\end{itemize}
\end{prop}

\subsubsection{Two Grounded Legs}
We consider a double support phase where legs $1$ and $2$ are grounded. Equation \eqref{eq:a:cd_combined} becomes
\begin{equation}
    m (\vec p_1 - \vec p_B) \times (\vec {\ddot p}_B - \vec g) +
(\vec p_2 - \vec p_1) \times \vec f_2  = \vec {\dot L}_B. 
\end{equation}
The projection along $\vec p_{12}$ eliminates $\vec f_2$
\begin{equation} \label{eq:a:line_constraint}
    m\det(\vec p_{12}, \vec p_B-\vec p_1, \vec g-\vec {\ddot p}_B) - \vec p_{12}^T \cdot \vec {\dot L}_B = 0.
\end{equation}
Due to the projection axis, we also need to limit the \ac{GIA} vector perpendicular to $\vec p_{12}$. If we look at the moment $\vec M_1 = (\vec p_B - \vec p_1) \times (\vec g - \vec {\ddot p}_B) - \vec {\dot L}_B/m$ induced about foothold $\vec p_1$, perpendicular to $\vec p_{12}$, then $ \vec e_z^T \cdot (\vec p_{12} \times \vec M_1) \geq 0$, or,
\begin{equation} \label{eq:a:double_support_generalized_zmp}
  \det(\vec e_z, \vec p_{12}, \vec M_1) \geq 0.
\end{equation}
Notice that the formulation requires an over-hanging configuration (\highlight{assumption~\ref{assumtopns:models:6})}.\footnote{This requirement comes from the fact that the projection axis is chosen as $\vec e_z$. A more general formulation is possible by replacing $\vec e_z$ with $\vec p_B - \vec p_1$, but would lead to a higher order of non-linearity.}
For the second foothold, condition~\eqref{eq:a:double_support_generalized_zmp} can only be satisfied if $\vec e_z^T \cdot \vec p_{12} \times (\vec p_{12} \times \vec f_2)  \leq 0$. Considering again the worst case scenario, i.e.,  $\mu_i=0$, the inequality constraint can be simplified to 
\begin{equation}
\begin{aligned}
    \vec e_z^T \cdot \vec p_{12} \cdot (\vec p_{12} \vec n_2)               &\leq \vec e_z^T \cdot \vec n_2 \cdot ||\vec p_{12}||^2 \\
    \vec e_z^T \cdot \vec p_{12} \cdot ||\vec n_2||\cos\theta               &\leq \vec e_z^T \cdot \vec n_2 \cdot ||\vec p_{12}|| \\
    \vec e_z^T \cdot \vec e_{p} \cos\theta                                  &\leq \vec e_z^T \cdot \vec e_n \\
    \vec e_{p,z} \cos\theta                                                 &\leq \vec e_{n,z},
\end{aligned}
\end{equation}
where we have used $\theta = \angle\{\vec p_{12},\,\vec n_2\}$, $\vec e_p = \vec p_{12} / ||\vec p_{12}||$ and  $\vec e_n = \vec n_{2} / ||\vec n_{2}||$. Since $\vec e_{n,z} \geq 0$ for a push contact (assumption~\ref{assumtopns:models:3}), the condition further simplifies to $\angle\{\vec p_{12},\,\vec n_2\} \geq 0$ and we can conclude that:
\begin{prop} \label{a:dyn_prop2}
There exists a set of contact forces $\{\vec f_1, \vec f_2\}$ satisfying \eqref{a:sbcd} and \eqref{a:no_slip_condition}, if
\begin{itemize}
    \item the \ac{GIA} vector is bounded by
    \begin{subequations}
    \begin{gather}
       m \det(\vec p_{12}, \vec p_B-\vec p_1, \vec g-\vec {\ddot p}_B) - \vec p_{12}^T \cdot \vec {\dot L}_B = 0 \label{eq:models:ss1} \\
        \det(\vec e_z, \vec p_{12}, \vec M_1) \geq 0 \label{eq:models:ss2},
    \end{gather}
    \end{subequations}
    \item the \ac{GIA} vector is bounded by the friction cone~\eqref{eq:models:simplified_no_slip_condition},
    \item \highlight{assumption~\ref{assumtopns:models:2} and~\ref{assumtopns:models:6} are satisfied}, and
    \item $\angle\{\vec p_{ij}, \vec n_j\}\geq 0,~\forall i\neq j$.
\end{itemize}
\end{prop}

Fig. \ref{fig:a:failure_case} demonstrates an exotic example where prop.~\ref{a:dyn_prop2} is violated.
It should be noted that prop~\ref{a:dyn_prop1} and \ref{a:dyn_prop2} do not exclude convex nor concave surfaces. For instance, Fig.~\ref{fig:a:success_case} shows a valid example, where horizontal forces cancel each other out.

\begin{figure}%
\centering
 \subfloat[Cone stability criterion fails.]{\includegraphics[width=0.44\columnwidth]{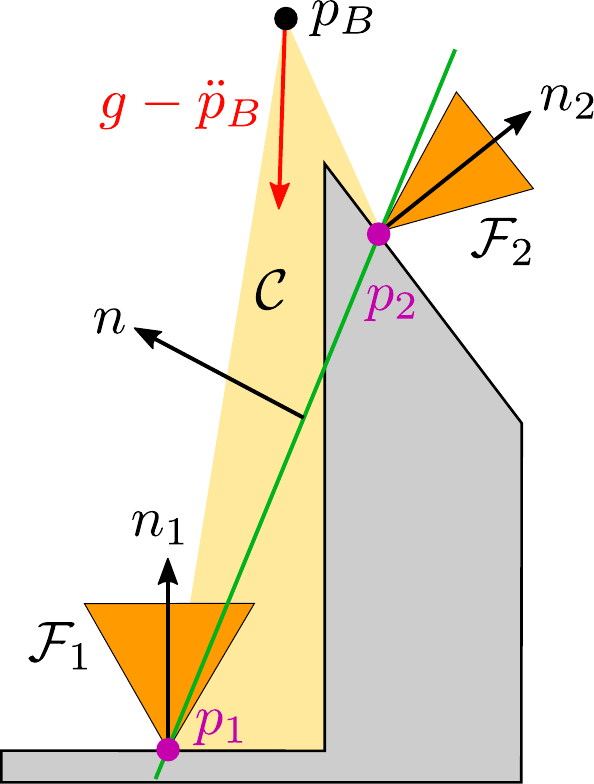} \label{fig:a:failure_case}}%
 \quad
 \subfloat[Cone stability criterion applies.]{\includegraphics[width=0.45\columnwidth]{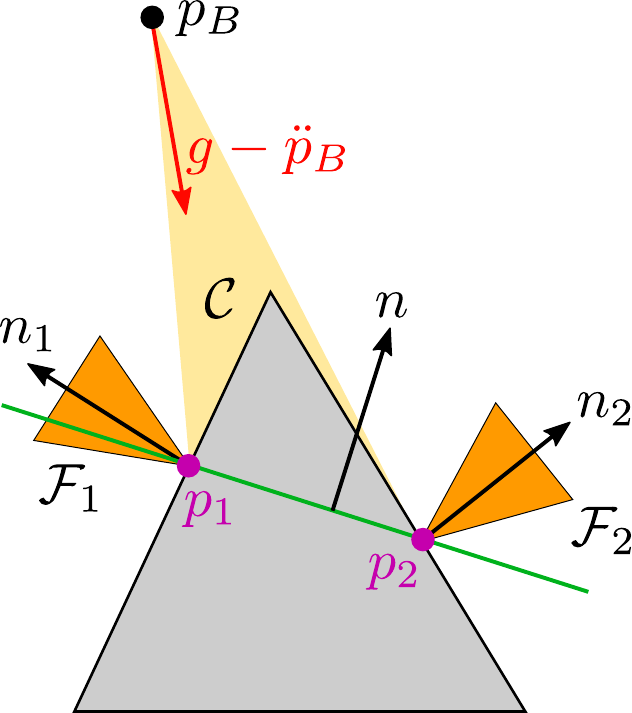} \label{fig:a:success_case}}%
 \caption{Left: We can find $\vec g - \vec {\ddot p}_B$ s.t. it is bounded by the \ac{GIAC} $\mathcal{C}$ and by the convex hull of friction cones $\{\mathcal{F}_1, \mathcal{F}_2 \}$. But we cannot find a contact force $\vec f_2\in\mathcal{F}_2$ that is lower bounded by the vector $(\vec p_2 - \vec p_1)$. Right: We can place the footholds $\vec p_1$ and $\vec p_2$ anywhere on the pyramid without ever violating proposition~\ref{a:dyn_prop2}.}
\end{figure}

\subsubsection{Single Contact and Full Flight Phase} 
In the case of a single contact, the 2D cone described by \eqref{eq:models:ss1} and \eqref{eq:models:ss2} simplifies to a line and the induced moment needs to aligned with it, i.e, $m(\vec p_B - \vec p_i) \times \vec a_B - \vec {\dot L}_B = \vec 0$.
Similarly, the base is constrained to evolve along the ballistic trajectory given by $\vec {\ddot p}_B = \vec g$ and $\vec {\dot L}_B = \vec 0$ during any full flight phase.

\subsection{Model Formulation} \label{ss:giac_summary}
Combining prop.~\ref{a:dyn_prop1} and~\ref{a:dyn_prop2} together with assumptions~\ref{ss:model_assumptions}, we can now state our dynamic stability criterion:
\begin{subequations}
\begin{align}
    \mu \vec e_z^T \cdot \vec a_B  \geq ||(\mathbb{I}_{3\times 3} - \vec e_z \vec e_z^T) \cdot \vec a_B || &&N&>0 \label{eq:models:m1}\\
    m \det(\vec p_{ij}, \vec p_B-\vec p_i, \vec a_B ) \leq \vec p_{ij}^T \cdot \vec {\dot L}_B&& N&\geq3 \label{eq:models:m2}  \\
    m \det(\vec p_{ij}, \vec p_B-\vec p_i, \vec a_B) = \vec p_{ij}^T \cdot \vec {\dot L}_B&& N &= 2 \\
    \det(\vec e_z, \vec p_{ij}, \vec M_i) \geq 0&& N &= 2 \\
    m(\vec p_B - \vec p_i) \times \vec a_B - \vec {\dot L}_B = \vec 0&& N &= 1 \label{eq:models:m5} \\
    \vec a_B = \vec 0 \quad \vec {\dot L}_B = \vec 0&& N &= 0. \label{eq:models:mN}
\end{align}
\end{subequations}

Due to its geometric interpretation discussed in~\ref{ss:giac_properties}, we refer to the set of constraints~\eqref{eq:models:m1} to~\eqref{eq:models:mN} as the \emph{\ac{GIAC} model}.

In section~\ref{ss:wbc:gm_observer} we recapitulate a method for estimating an external base wrench disturbance $\vec {\hat w}_{\text{ext},B} = \Mx{ \vec {\hat f}_B, \vec {\hat \tau}_B }$. By assuming the wrench to be constant over the prediction horizon, we can complete the dynamic constraints, e.g., \eqref{eq:models:m2} can be written as
\begin{equation} \label{eq:models:zmp_3legs}
       \det\left(\vec p_{ij}, \vec p_B-\vec p_i, \vec a_B + \frac{\vec {\hat f}_B}{m}\right) 
       \leq \vec p_{ij}^T \left(\frac{\vec {\dot L}_B }{m} - \vec {\hat \tau}_B \right).
\end{equation}

\subsection{Properties}  \label{ss:giac_properties}
In the appendix~\ref{a:comparison} we proof the following:
\begin{itemize}
    \item The \ac{GIAC} constraints~\eqref{eq:models:m2} to~\eqref{eq:models:mN} can be interpreted as the largest inscribing convex cone $\tilde{\mathcal{C}}$ of the so called \emph{\acf{GIAC}} $\mathcal{C}$, which is defined as the convex hull of rays connecting the base position with the footholds. Rate of change of the angular momentum changes shape and size of these cones~(\ref{ss:giac_interpretation}).
    \item The \ac{ZMP} is the projection of the \ac{GIA} vector onto the ground. On flat ground, the \ac{GIAC} constraints~\eqref{eq:models:m2} to~\eqref{eq:models:mN} simplify to the well known \ac{ZMP} stability criterion~(\ref{a:comparison_zmp}).
    \item Under assumption~\ref{ss:model_assumptions}, the \ac{GIAC} model is \emph{weak contact stable}, a property shared with the closely related \ac{CWC} models~(\ref{a:comparison_cwc}).
\end{itemize}

\section{Perception} \label{s:Perception}
\begin{figure}
\centering
\includegraphics[width=1.0\columnwidth]{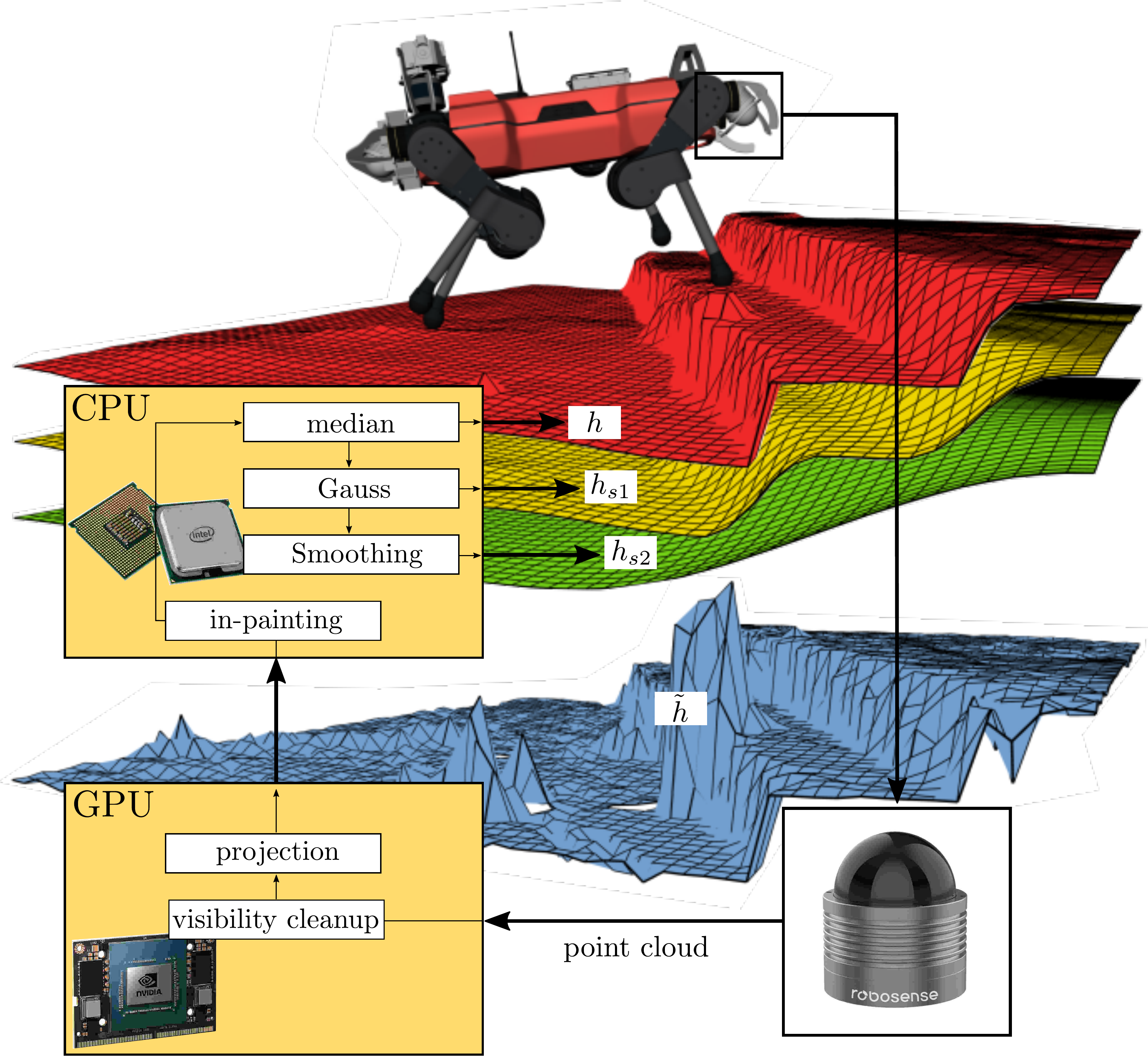}
\caption{Illustration of the mapping pipeline using data recorded in a real-world experiment. The raw map $\tilde h$ (blue) is generated on GPU and updated at \unit[20]{Hz}. A filter chain, implemented on CPU, removes noise and artifacts.}
\label{fig:planning:perception}
\end{figure}
Next to proprioceptive sensors, which allow the robot to perceive its surroundings by touch, vision is the second most important sense for locomotion. It allows the robot to not only react to the environment (blind) but also to plan its motion ahead w.r.t. the terrain topology (perceptive). 
The perception pipeline is illustrated in Fig.~\ref{fig:planning:perception}. Two onboard LiDARs are used to generate a detailed \emph{height map}~\cite{Fankhauser2014}, abbreviated by $\tilde h$. For performance reasons, projection and visibility clean-up are performed on an onboard GPU~\cite{fabianje2020}. Further map processing algorithms run on CPU and generate three height layers of different smoothness, $h$, $h_{s1}$, and $h_{s2}$. 

\subsection{Filtering}
We deploy a filter chain consisting of in-painting, outlier rejection, and smoothing.
1) First, we iterate over all grid cells in the raw map $\tilde h$, replacing empty cells with the minimum found across the occlusion border.
2) Reflections, damaged lenses, odometry drift, or a miss-aligned URDF are known sources of artifacts. These outliers are removed by a sequentially repeated median filter~\cite{medianFilter} on the inpainted map.
3) Given the de-noised map $h$, we compute two additional layers, $h_{s1}$ and $h_{s2}$. The former is a slightly Gaussian filtered version of the original map and is used to compute gradients of edges
\begin{equation}
    h_{s1} = \text{Gauss}_{\sigma_1}(h).
\end{equation}
The layer $h_{s2}$ represents a ``virtual floor'' and is derived in four steps. The map $h$ is aggressively median filtered and then subtracted from itself
\begin{equation}
\Delta h = h - \text{median}(h)   
\end{equation}
The height difference is used to set up a filter mask $m$,
\begin{equation}
    m_{ij} = 
    \begin{cases}
    1 & \Delta h_{ij}>0 \text{ (stepping stone)}\\
    1 & \Delta h_{ij}<0 \text{ (gap)} \\ 
    \infty & \text{otherwise}.
    \end{cases}
\end{equation}
In the third step, we widen stepping stones and narrow gaps through masked dilation. This filter replaces height values of $h$ by the maximum found in its close neighborhood. Candidate grid cells must be categorized as either stepping stones or gaps,
\begin{equation}
    h_\text{dilated} = \text{dilate}(h \mid m = 1).
\end{equation}
Finally, the processed height map is smoothed by a Gaussian filter with a large standard deviation $\sigma_2 > \sigma_1$
\begin{equation}
h_{s2} = \text{Gauss}_{\sigma_2}(h_\text{dilated}).
\end{equation}
Such a four-staged filter hierarchy is, in most cases, just identical to Gaussian smoothing. However, in the presence of positive/negative obstacles, more weight is naturally given to elevated parts, as exemplified in Fig.~\ref{fig:planning:stepping_stones}.

Post processing relies on \texttt{opencv}~\cite{opencv_library} filters which are deployed on CPU. For a grid map with dimensions \unit[2.4]{m} $\times$ \unit[2.4]{m} and a cell size \unit[4]{cm} $\times$ \unit[4]{cm}, the post-processing step takes about \unit[1.5]{ms}. 
\begin{figure}
\centering
\includegraphics[width=1.0\columnwidth]{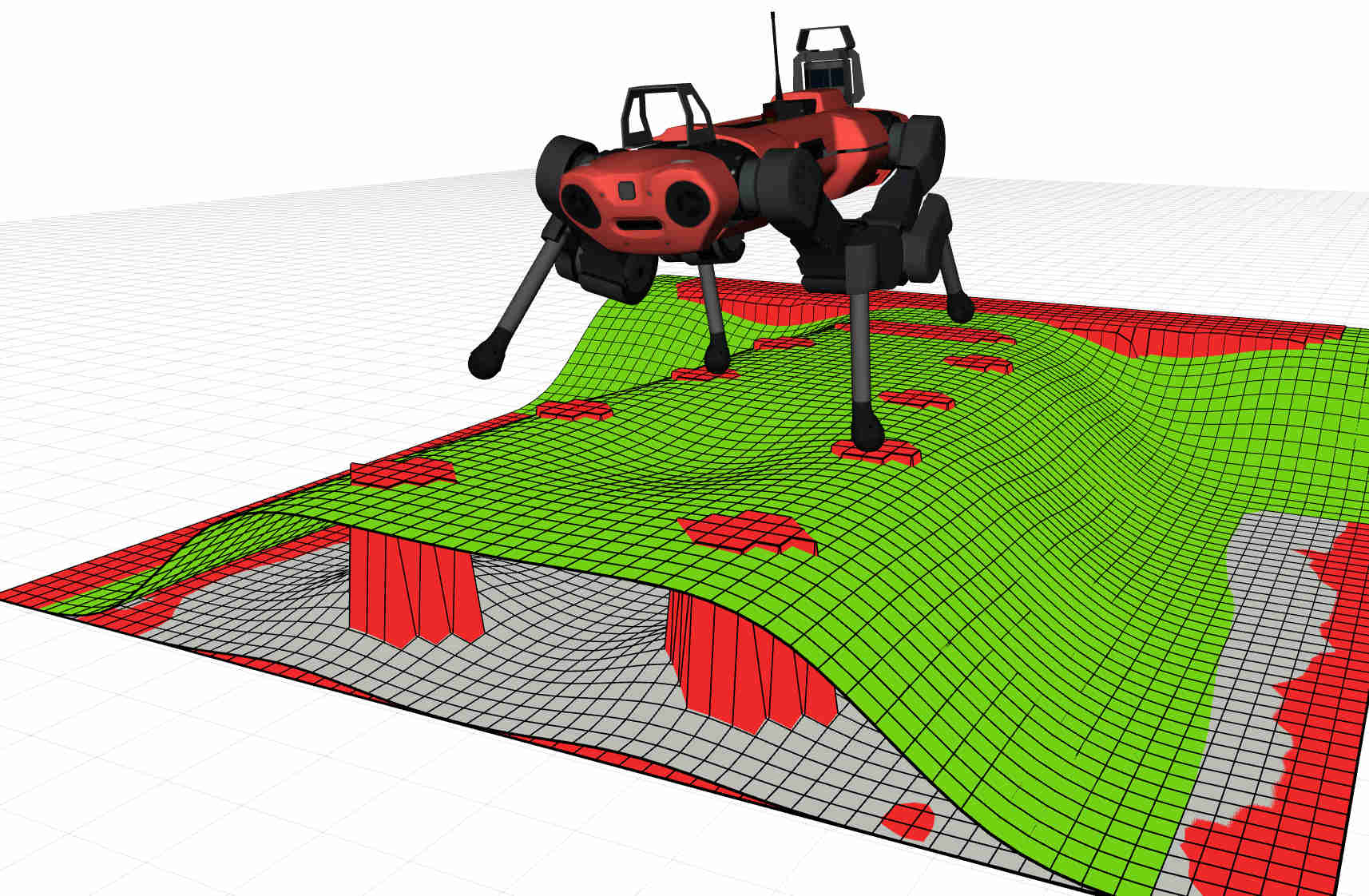}
\caption{The red map $h$, representing geometric features of the terrain, is used for foot placement, while the green map $h_{s2}$ serves as a base pose reference. In contrast to the gray layer, which would be obtained from $h$ by pure Gaussian smoothing, $h_{s2}$ spawns above the stepping stones. This encourages the robot to walk along the path indicated by the cylindrical pillars.}
\label{fig:planning:stepping_stones}
\end{figure}

\subsection{Numerical Differentiation and Interpolation}
We use a $1$D $5$-point central finite difference kernel for computing numerical derivatives
\begin{align}
    &\text{1th order} &&\frac{1}{12 \Delta x}\Mx{1&-8&0&8&-1} \\
    &\text{2th order} &&\frac{1}{12 \Delta x^2}\Mx{-1&16&-30&16&-1},
\end{align}
with $\Delta x$ the grid cell length and width.
Smoothing perpendicular to the derivative direction is not necessary, as the height maps $\{h, h_{s1}, h_{s2}\}$ are already relatively smooth. This reduces run time compared to classical $2$D derivative kernels as the number of multiplications and additions decreases by a factor of five. Gradients and curvatures are computed online whenever needed and stored in a look-up table for potential re-use. As the optimizer typically only visits a tiny fraction of grid cells, this approach is faster than pre-computing the convolutions for all layers a priory.

Since the height derivatives are computed on a discrete net, solutions to gradient-based solvers become restricted to the grid map's spatial resolution. This problem can be most prominently observed as jumps of the footholds between neighboring grid cells across two consecutive optimizations. We tackle this problem using a bi-linear interpolation scheme: Derivative values are interpolated between four neighboring grid cells.

\section{Motion Optimization} \label{s:optimization_method}
We parametrize the base pose trajectory as a $6$D spline of fixed order five. The first three dimensions capture position, and the last three dimensions describe orientation using Euler angles.
To increase the feasible space, we allow the base acceleration to evolve discontinuous at contact transitions~\cite{Kalakrishnan1, BellicosoJenelten2017}. We do so by breaking the trajectory into different spline segments connected with each other at the transition times. For the $k$th spline we write the $l$th dimension as
\begin{equation}
    \vec \Pi_{B,kl}(t) = a_{0kl} + a_{1kl}t + \ldots + a_{4kl}t^4,
\end{equation}
with $t \in(0, \tau_k)$, $0\leq k <N_s-1$, $0 \leq l < 5$, where $a_{0kl},\ldots, a_{4kl}$ are the spline coefficients, $N_s$ is  the number of splines and $\tau_k$ is the duration of the $k$th spline.
The method is further explained in Fig.~\ref{fig:planning:collocation}.

\begin{figure}
\centering
\includegraphics[width=1.0\columnwidth]{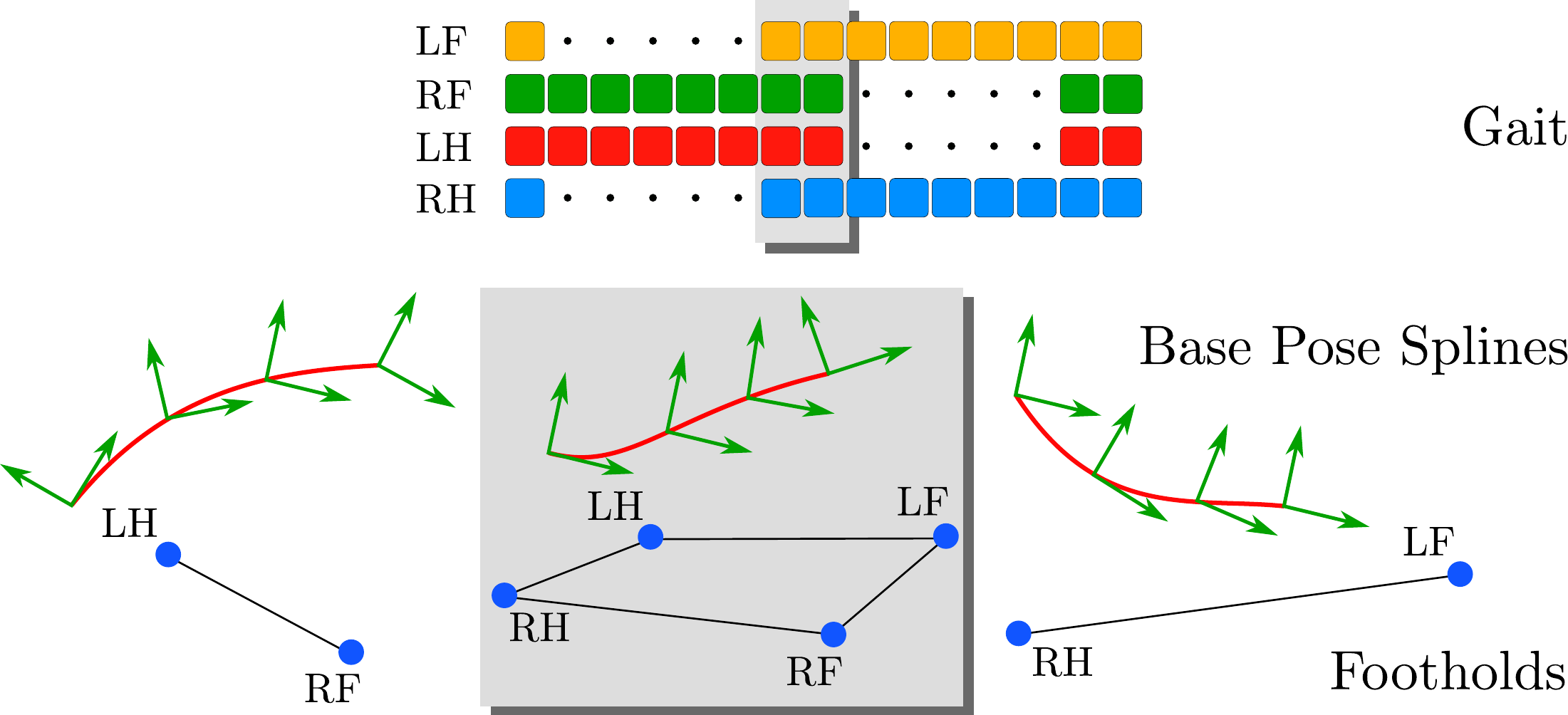}
\caption{The gait pattern is fixed and defines lift-off and touch-down timings for each leg. Two adjacent phase events enclose a time interval of constant contact states. Each phase is assigned a $6$D base pose spline and a set of active footholds. The prediction horizon includes one step per leg.}
\label{fig:planning:collocation}
\end{figure}

We stack all unknown variables \highlight{(which are spline coefficients $a_{0kl}, \ldots, a_{4kl}$, positions of grounded feet $\vec p_\text{meas}^T= \Mx{\vec p_{\text{meas},1}^T & \cdots & \vec p_{\text{meas},N}^T}$, desired footholds $\vec p^T = \Mx{\vec p_1^T & \cdots & \vec p_4^T}$, and slack variables $\vec \varepsilon^T$)} together into a state vector $\vec x$, 
\begin{equation}
    \vec x = \Mx{a_{000} & \cdots & a_{4(N_s-1)5} & \vec p_\text{meas}^T & \vec p^T & \vec \varepsilon^T}^T
\end{equation}
\highlight{For the reminder of this section, if not stated differently, we will no more differentiate between foot position and foothold, and refer to both as foothold $\vec p_i$.}

The constraint \ac{NLP} can be formulated as
\begin{equation} \label{eq:method:opt}
\begin{aligned}
    \min_{\vec x}~ &\sum_i f_i(\vec x) \\
    \text{s.t.} \quad \vec c_\text{eq}(\vec x) &= \vec 0,
    \quad \vec c_\text{ineq}(\vec x) &\leq \vec 0,
\end{aligned}
\end{equation}
where individual objective functions $f_i$, equality constraint $\vec c_\text{eq}$, and inequality constraints $\vec c_\text{ineq}$ are detailed in the following two subsections.

\subsection{Constraints}
\subsubsection{Initial and Junction Constraints}
Initial constraints are formulated for the base pose $\vec \Pi_B(0)$, base twist $\vec{\dot \Pi}_B(0)$, and foot positions $\vec p_{\text{meas},i}$ of grounded legs $i$. Junction constraints are imposed between two adjacent splines to ensure smoothness up to the first derivative, i.e, $\vec \Pi_{B,k+1}(0) = \vec \Pi_{B,k}(\tau_k)$ and $\vec {\dot \Pi}_{B,k+1}(0) = \vec {\dot \Pi}_{B,k}(\tau_k)$.

\subsubsection{Dynamic Stability} \label{ss:planning:dyn}
The dynamic constraints~\eqref{eq:models:m2} to \eqref{eq:models:mN} are implemented as slacked inequalities,
\begin{equation} \label{eq:planning:dyn}
        \min~ \varepsilon
        \qquad\text{s.t.}\quad \vec c_\text{dyn}(\vec x) \leq  \vec 1 \varepsilon,
        \quad  \varepsilon  \leq 0.
\end{equation}
Notice that constraints of type equality can always be transformed to inequalities by $\vec c_\text{dyn}(\vec x) = \Mx{\vec c_\text{dyn, eq}(\vec x) & -\vec c_\text{dyn, eq}(\vec x)}^T$. 

The slack variable $\varepsilon$, which is unique to each phase, maximizes the robustness margin, i.e., the smallest angle between the \ac{GIA} vector and the bounding cone. Choosing a large weight will increase the \ac{GIAC} volume by pushing the footholds away from each other.

\subsubsection{Convex \acs{GIAC}}
\begin{figure}
\centering
\includegraphics[width=1.0\columnwidth]{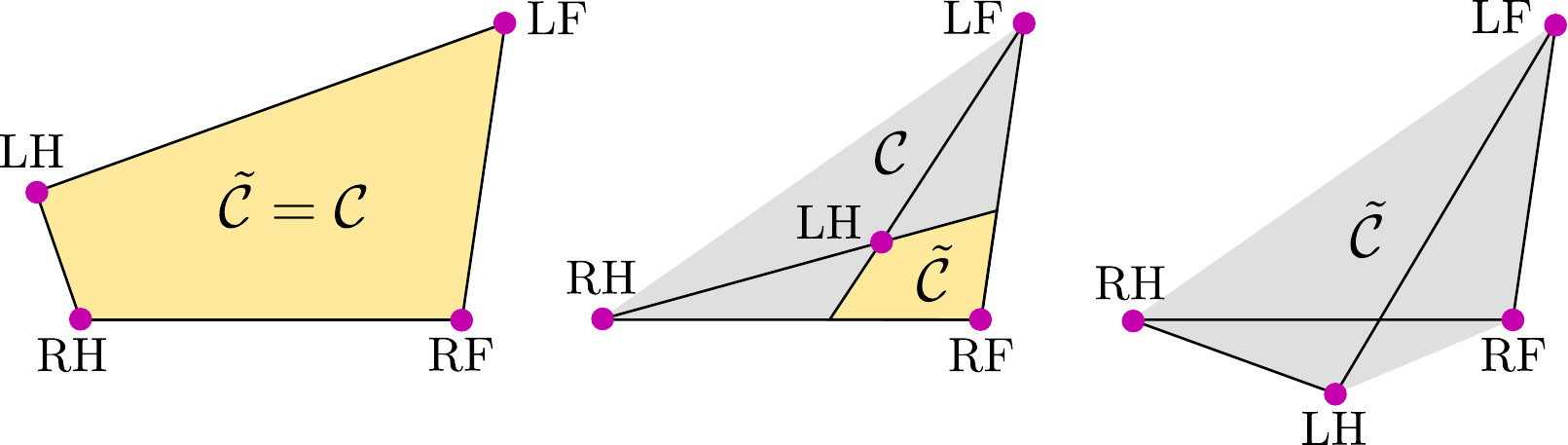}
\caption{Gray: \ac{GIAC} $\mathcal{C}$ as seen from the top. Yellow: Largest inscribing convex cone $\tilde{\mathcal{C}}$, defining the feasible space of the \ac{GIA} vector. If the cone $\mathcal{C}$ is convex, then $\tilde{\mathcal{C}}= \mathcal{C}$ (left), otherwise $\tilde{\mathcal{C}} \subset \mathcal{C}$ (middle), or $\tilde{\mathcal{C}} = \emptyset$ (right).}
\label{fig:planning:convex_cone}
\end{figure}
The feasible space imposed by the dynamic stability criterion is given by the cone $\tilde{\mathcal{C}} \subseteq \mathcal{C}$ with $\mathcal{C}$ being the \ac{GIAC}. If $\mathcal{C}$ is non-convex, then $\tilde{\mathcal{C}} \subset \mathcal{C}$, and the feasible space shrinks accordingly. Fig.~\ref{fig:planning:convex_cone} exemplifies three possible scenarios:  No, partial and complete reduction of the support volume. By enforcing convexity of $\mathcal{C}$, we can avoid the reduction of the effective feasible volume, and thereby guarantee that prop.~\ref{prop:models:zmp_linear} and constraints~\eqref{eq:models:m2} to \eqref{eq:models:mN} are identical. This can be achieved by constraining the footprint to a convex shape, and the kinematic configuration to an over-hanging torso, i.e., $\vec e_z^T \cdot \vec p_B \geq \vec e_z^T \cdot \vec p_i$ \highlight{(assumption~\ref{assumtopns:models:6})}. For a phase with four grounded feet \highlight{$\{1,2,3,4\}$}, counterclockwise ordered, the former constraint can be imposed as
\begin{equation} \label{eq:planning:convex_support}
\begin{aligned}
    (\vec p_{13} \times \vec p_{12}) \cdot \vec e_z & \leq 0 &
    (\vec p_{14} \times \vec p_{13}) \cdot \vec e_z & \leq 0 \\
    (\vec p_{24} \times \vec p_{23}) \cdot \vec e_z & \leq 0 &
    (\vec p_{20} \times \vec p_{24}) \cdot \vec e_z & \leq 0.
\end{aligned}
\end{equation}
We also apply~\eqref{eq:planning:convex_support} in the absence of a full stance phase: We iterate over future phase events and add corresponding footholds until a set of four is complete. This reduces the risk of swing-leg collisions, as explained in Fig.~\ref{fig:planning:covex_footholds1} and~\ref{fig:planning:covex_footholds2}.
\begin{figure*}%
\centering
 \subfloat[Unconstrained footprint: \ac{LF} overtakes \ac{RF}, leading to a collision of the two knee joints. Notice how the kinematic configuration is perfectly symmetric during the full stance phase.]{\includegraphics[width=0.33\textwidth]{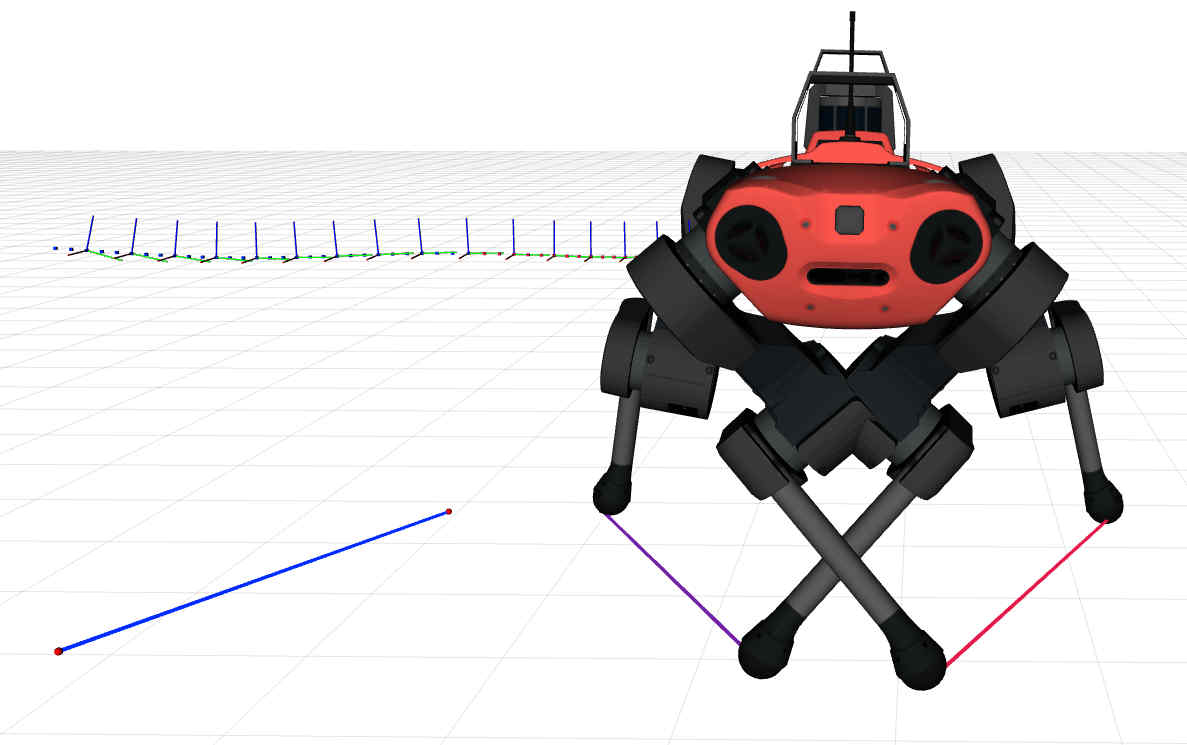} \label{fig:planning:covex_footholds1}}%
 \quad
 \subfloat[Enforcing convex footprint and collision avoidance: \ac{LF}, \ac{RF} and \ac{LH} join on a line. The realized velocity is considerably smaller and the kinematic symmetry is  perturbed.]{\includegraphics[width=0.29\textwidth]{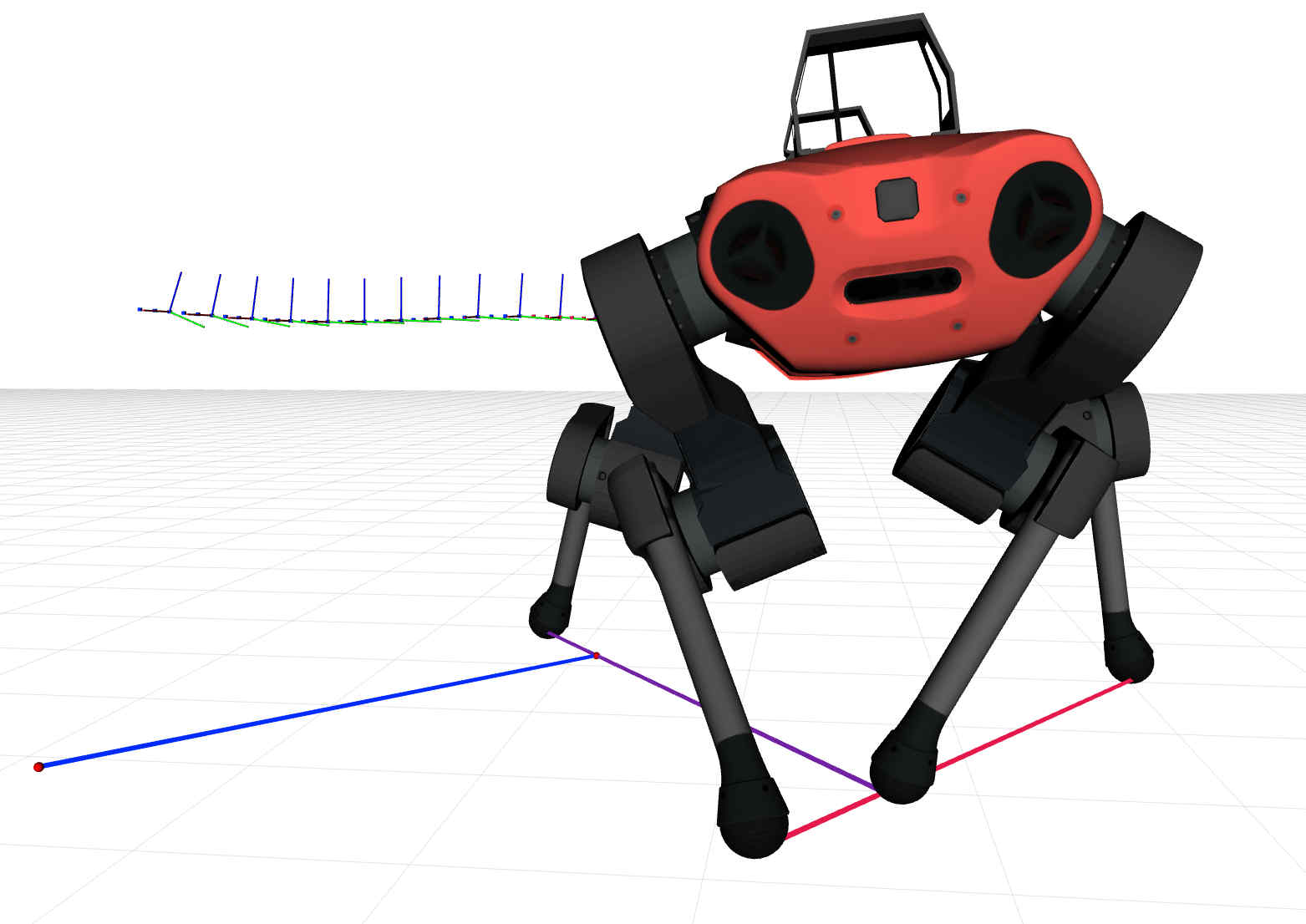} \label{fig:planning:covex_footholds2}}%
 \quad
 \subfloat[The constraint for footprint convexity is active, but leg collision constraints are disabled. The two front legs are likely to step on each other.]{\includegraphics[width=0.3\textwidth]{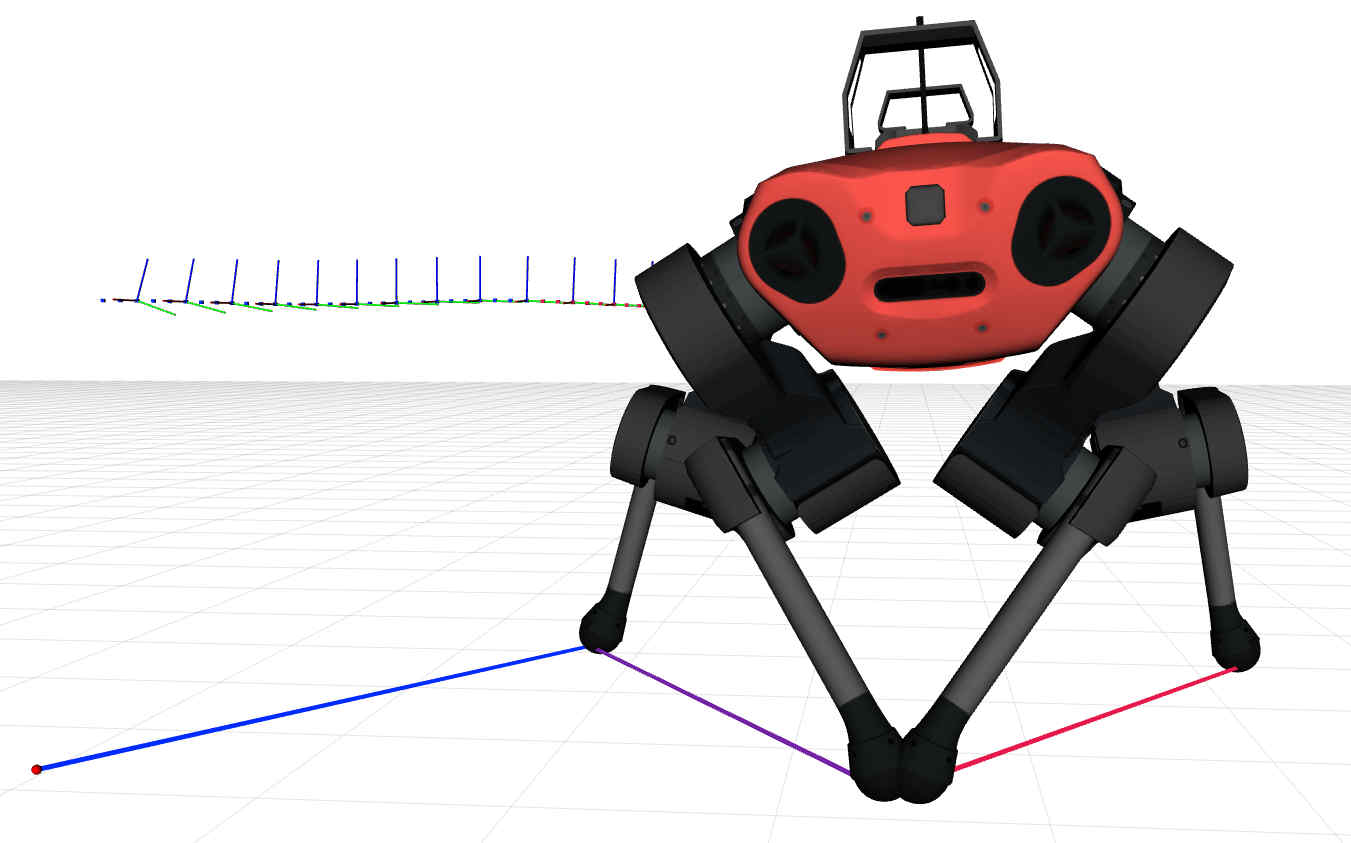} \label{fig:planning:covex_footholds3}}%
 \caption{Illustration of foothold tasks for a trotting gait on flat ground. The lateral reference velocity was \unit[1.5]{m/s} for all three simulation experiments. For our \ac{TO} method, we use the second set-up, which is less likely to encounter feet or knee joint collisions.} \label{fig:planning:task_convex_footholds}
\end{figure*}

\subsubsection{Friction Cone}
The no-slip condition is embedded in the weak and simplified form~\eqref{eq:models:simplified_no_slip_condition}. An additional constraint is required to ensure that the robot pushes \highlight{(assumption~\ref{assumtopns:models:3})} on the ground, i.e., $\vec e_z^T \cdot \vec {\ddot p}_B \geq \vec e_z^T \cdot \vec g$.

\subsubsection{Kinematics} \label{ss:planning:kin_constraints}
We approximate the reachable space of an end-effector $i$ by two balls, centered at the rotation center of the hip\footnote{In case of ANYmal, this point can be found where the rotation axis of the \ac{HAA} and \ac{HFE} joint intersect.}
\begin{equation} \label{eq:planning:kinematics}
l_\text{min}^2 \leq ||\vec p_B + \mat R_B \vec r_i^B - \vec p_i||^2 \leq l_\text{max}^2.
\end{equation}
The vector $\vec r_i^B$ specifies the limb center and $l_\text{max}$/$l_\text{min}$ denotes the maximal/minimal limb extension.

\subsection{Objectives}
\subsubsection{Footholds on Ground} \label{ss:planning:footholds_on_ground}
To ensure that a desired foothold $\vec p_i$ is located on the height map $h$, we may incorporate the constraint $h(\vec p_i) = \vec e_z^T \cdot \vec p_i$. As Fig.~\ref{fig:planning:local_opt} exemplifies, such a formulation can lead to a suboptimal solution or even local infeasibility. We therefore soften the task and write $\min \big(h(\vec p_i) - \vec e_z^T \cdot \vec p_i \big)^2$.

\subsubsection{Leg Collision Avoidance}
During fast lateral motions, foot collisions with neighboring legs are likely to occur. We wish to enforce a lower bound $\epsilon_\text{min}$ on the Euclidean distance in the $xy$ plane between any two feet $i$ and $j$, for instance
\begin{equation}
    z(\vec p_i, \vec p_j) = ||(\mat I - \vec e_z \vec e_z^T)\cdot (\vec p_i - \vec p_j)||^2 \geq \epsilon_\text{min}.
\end{equation}
Unfortunately, this constraint is concave and not well suited for gradient-based optimization. Instead, we use a one-sided quadratic barrier approximation~\cite{gaertner2021collisionfree},
\[
\min~
\begin{cases}
(\epsilon_\text{min} - z(\vec p_i, \vec p_j))^2 & z < \epsilon_\text{min} \\
0            & z   \geq \epsilon_\text{min}.
\end{cases}
\]
Fig.~\ref{fig:planning:covex_footholds2} illustrates the impact of the collision avoidance task on the solution.

\subsubsection{Nominal Kinematics}
We want to relate a set of nominal footholds with a nominal base pose. By introducing a desired leg extension vector $\vec l_\text{des} = \Mx{ 0 & 0 & h_\text{des}}^T$ with $h_\text{des}$ the desired height above ground, we can write
\begin{equation} \label{eq:planning:nom_kin}
    \min~||\vec p_B + \mat R_B \vec r_i^B - \vec p_i - \vec l_\text{des}||^2.
\end{equation}
For each foothold, the objective is applied once per stance. More specifically, we define this time as midstance, which introduces the symmetric behavior observed in Fig.~\ref{fig:planning:covex_footholds1}.

\subsubsection{Base Pose Alignment} \label{ss:planning:base_pose_alignment}
A popular approach used in blind quadrupedal locomotion is to align the torso with some local terrain estimate~\cite{Gehring2016,BellicosoJenelten2017,Bellicoso2}, but similar ideas are also used for perceptive planners~\cite{fabianje2020, Ruben3}. The terrain is typically represented by a plane, obtained by a least-squares fit from current and previous stance foot locations~\cite{Gehring2016}. 

We choose a more general approach and invoke 
\begin{equation} \label{eq:planning:base_pose_alignment}
\min~\sum_{i = 1}^4 \Big[\vec e_z^T \cdot (\vec p_B + \mat R_B \vec r_i^B - \vec l_\text{des}) - h_{s2}(\vec p_B) \Big]^2.
\end{equation}
The local plane is somewhat implicitly given by the spatial locations of limb thighs $\{\vec p_B + \mat R_B \vec r_i^B\}_{\highlight{i=1,\ldots, 4}}$, and objective~\eqref{eq:planning:base_pose_alignment} resembles the least squares fit.
The smooth map $h_{s2}$ is preferred over $h$, as it ``smooths out'' non-convexity.

\subsubsection{Edge Avoidance} \label{ss:planning:edge_avoidance}
The robot shall avoid stepping on edges to minimize the risk of slippage. Methods that rely on plane segmentation avoid this problem naturally for point feet~\cite{Ruben3}, or at least give rise to certain simplifications for humanoid feet~\cite{Griffin2019}. An often encountered idea merges geometric features into a \emph{foothold score} and selects footholds in a way that minimizes this quantity. Such cost functions are widely used in \ac{MMO} structures~\cite{Kolter1, Kalakrishnan1, Fankhauser2018, Magana2019, fabianje2020}.  

Plane segmentation and foothold scores are computationally expensive. Moreover, for gradient based methods, the foothold score needs to be at least locally convex\footnote{There needs to be a gradient that pushes the foothold in the ``good'' direction.} and conventional traversability maps or binary scores do not apply. Instead, we suggest to penalize directly height gradient
\begin{equation}
\min~\nabla h(\vec p_i)^T \nabla h(\vec p_i) + \nabla h_{s1}(\vec p_i)^T \nabla h_{s1}(\vec p_i).
\end{equation}
We utilize the blurred version $h_{s1}$ in the second term to push the optimum away from edges. The larger the standard deviation of the Gaussian kernel is chosen, the farther away the feet will be placed from the edges, but the more details are blurred out. Both cost terms penalize inclination, which is in line with the horizontal contact plane \highlight{assumption~\ref{assumtopns:models:5}}.

\subsubsection{Previous Solution}
In the presence of external (e.g., slip) and internal (e.g., noisy elevation map, imperfect tracking) disturbances, desired footholds are likely to jump in between two adjacent optimization steps. And such jumps are particularly large if the ground is cluttered with edges or other infeasible foothold locations. To increase the step confidence, we minimize squared distance to the previous optimal foothold.

\subsubsection{Tracking} \label{ss:planning:tracking}
We track momentum in base frame by
\begin{equation}
    \min~\frac{||\vec P_B^B - \vec P_\text{des}^B||^2}{m^2}  + 
    ||\vec L_B^B - \vec L_\text{des}^B||^2.
\end{equation}
This task is identical to approach the planned twist towards a reference twist $\{ \vec{\dot p}_{B,\text{des}}^B,~ \vec{\omega}_{B,\text{des}}^B \}$ and using the weight $\mathbb{I}_{3\times 3}$ for the linear and the weight $\mat I_B^T \mat I_B$ for the angular velocity. 

\subsubsection{Smoothness} \label{ss:planning:min_diff_momentum}
In order to not diverge too far from \highlight{assumption~\ref{assumtopns:models:2},} we minimize the rate of change of angular momentum by $\min||\vec {\dot L}_B||^2$.

\subsection{Implementation Details}
\subsubsection{Solver}
By leveraging the \emph{\ac{GN}} method, as summarized in appendix~\ref{a:gauss_newton}, we can extract the positive definite part of a non-convex objective. We use a custom-made \ac{SQP} solver that internally approximates the problem as a sequence of \acp{QP}. The convexity of the resulting \acp{QP} is exploited by deploying an efficient implementation of \texttt{QuadProg++}~\cite{LucaDiGaspero1998}. 

At each iteration, a line search globalization approach~\cite{NoceWrig06} trades off cost function minimization and constraint violation to compute a search step length. As long as the linearization of the feasible set is non-empty, the \ac{SQP} solver is guaranteed to find a local solution. Since the tasks can be non-convex on a global scale, optimality can only be guaranteed locally.

\subsubsection{Automatic Differentiation}
Analytical derivatives of the tasks are sometimes difficult to derive by hand. They might also be numerically inefficient due to a large number of trigonometric multiplications. To this end, we deploy \texttt{CppADCodeGen}~\cite{CppADCodeGen}, a \texttt{C++} library for automatic differentiation and code generation.

\subsubsection{Trajectory sampling and Tuning}
Some tasks appear once per problem or once per spline segment, while others are enforced along the entire prediction horizon. These continuous tasks are sampled on a discrete grid and enforced for each $T_k = \tau_k/6$ seconds. Table~\ref{tab:weights} summarizes the weights used in \ac{TAMOLS} which are identical for all gaits. The weight for foothold task~\ref{ss:planning:footholds_on_ground} is chosen comparably large s.t. the constrained violation of a kinematically feasible problem stays within the mapping accuracy ($<$ \unit[5]{mm}). Significant constraint violations are very rare and only happen when no feasible solution exists around the initial guess.

\begin{table}
\centering
\caption{Weights for soft tasks. Continues-time tasks are sampled along the trajectory and marked as ``sampled''.}
\begin{tabular}{l c c}
\toprule
task                      & sampled     & weight  \\ \midrule
robustness margin         &             & $0.007$     \\
footholds on ground       &             & $10^{4}$     \\
leg collision avoidance   &             & $0.001$     \\
nominal kinematics        &             & $7$     \\
base pose alignment       & \checkmark  & $100 \cdot T_k$     \\
edge avoidance            &             & $3$     \\
previous solution         &             & $0.01$     \\
tracking                  & \checkmark  & $2\cdot T_k$     \\
smoothness                & \checkmark  & $0.001\cdot T_k$     \\ \bottomrule
\end{tabular}
\label{tab:weights}
\end{table}

\subsection{Initialization}\label{ss:planning:init_guess}
\begin{figure}
\centering
\includegraphics[width=0.80\columnwidth]{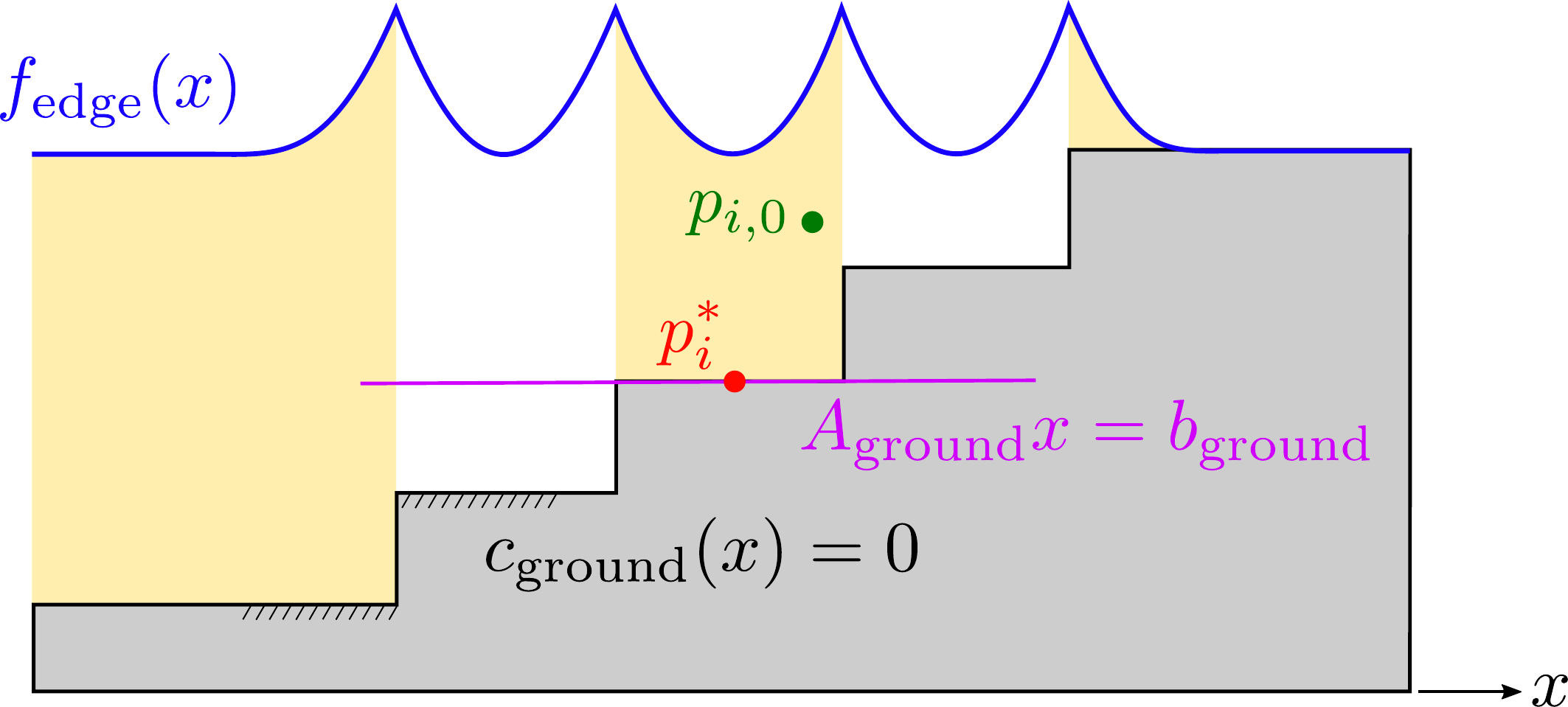}
\caption{Given an initial foothold $\vec p_{i,0}$, the optimizer will find a solution $\vec p_i^*$, which is most likely located on the same tread. The cost function for edge avoidance $f_\text{edge}$ increases on both ends of the tread, inducing a local optimum at its center. Moreover, the linearization of the height constraint $c_\text{ground} = 0$ is only valid on the tread where the constraint was linearized about. This renders exploration beyond the initial tread difficult because either the line search could reject such an update or the update could lead to an infeasible kinematic configuration in the next SQP iteration. A soft task allows finding feasible footholds by exploring unfeasible intermediate solutions.}
\label{fig:planning:local_opt}
\end{figure}
In contrast to most interior point methods, \ac{SQP} solvers can be efficiently warm started, in our case, using the latest optimal solution. However, at each leg touch-down, a new foothold and base pose spline segment appear, for which there is no previous solution available. Initializing the new states with some heuristics is prone to drive the solution into a local optimum. This issue is elaborated using the stair example in Fig.~\ref{fig:planning:local_opt}. 
Additional problems arise due to the \ac{GN} approximation, which disregards important second-order derivative information. For instance, the objective for edge avoidance~\ref{ss:planning:edge_avoidance} loses the curvature of the map, rendering the classification of extrema impossible. It is crucial to provide an initial guess located sufficiently far from the saddle or maximum points in this particular case.   

Even though indisputably important, the available literature often does not give insightful information into the initialization procedure. Some exceptions can be found in the data-driven community, where initial guesses are learned~\cite{Mansard2018,towr_learning,Melon2021}. 

\subsubsection{Batch Search Optimization}
In our previous work~\cite{fabianje2020}, we have introduced \emph{batch search}. The algorithm first establishes nominal footholds, based on Raibert's heuristic, around an approximate base pose trajectory. It then iterates within a circular search space centered about the nominal foothold, selecting a kinematically reachable grid cell with the lowest cost. Since footholds are selected independently from each other, and the base pose trajectory does not take into account the terrain topology, such an initial guess might lead to a suboptimal or even infeasible optimization problem.

\subsubsection{Graduated Optimization}
Instead of finding a more sophisticated initial guess capable of avoiding the aforementioned pitfalls, we deploy a global optimization technique: At each leg touch-down, we solve a sequence of \ac{TO} problems, starting with a greatly simplified problem and progressively approaching towards the original problem. Such a technique is generally known as \emph{graduated optimization}~\cite{Mobahi2015}.

In the case of \ac{TAMOLS}, non-convexity and discontinuity are introduced majorly by the terrain and the base orientation, but most prominently with the two foothold tasks~\ref{ss:planning:footholds_on_ground} and~\ref{ss:planning:edge_avoidance}. This observation gives rise to the idea of solving a ``nearly convex'' \ac{TO} in the first iteration by replacing the height maps $h$ and $h_{s1}$ with $h_{s2}$. The robot's measured state suffices already as an initial guess for this simplified \ac{NLP} to succeed. 

We could now continue in the sense of graduated optimization and introduce several virtual floors of different smoothness that gradually approach to the original map. For each of these virtual floors, a \ac{TO} problem needs to be solved, progressively adjusting the solution towards the actual terrain. Even though such an approach may be capable of finding the global solution, it would not be tractable due to the exploding number of optimization steps.

\subsubsection{Hierarchical Graduated Optimization}
\begin{figure}
\centering
\includegraphics[width=1.0\columnwidth]{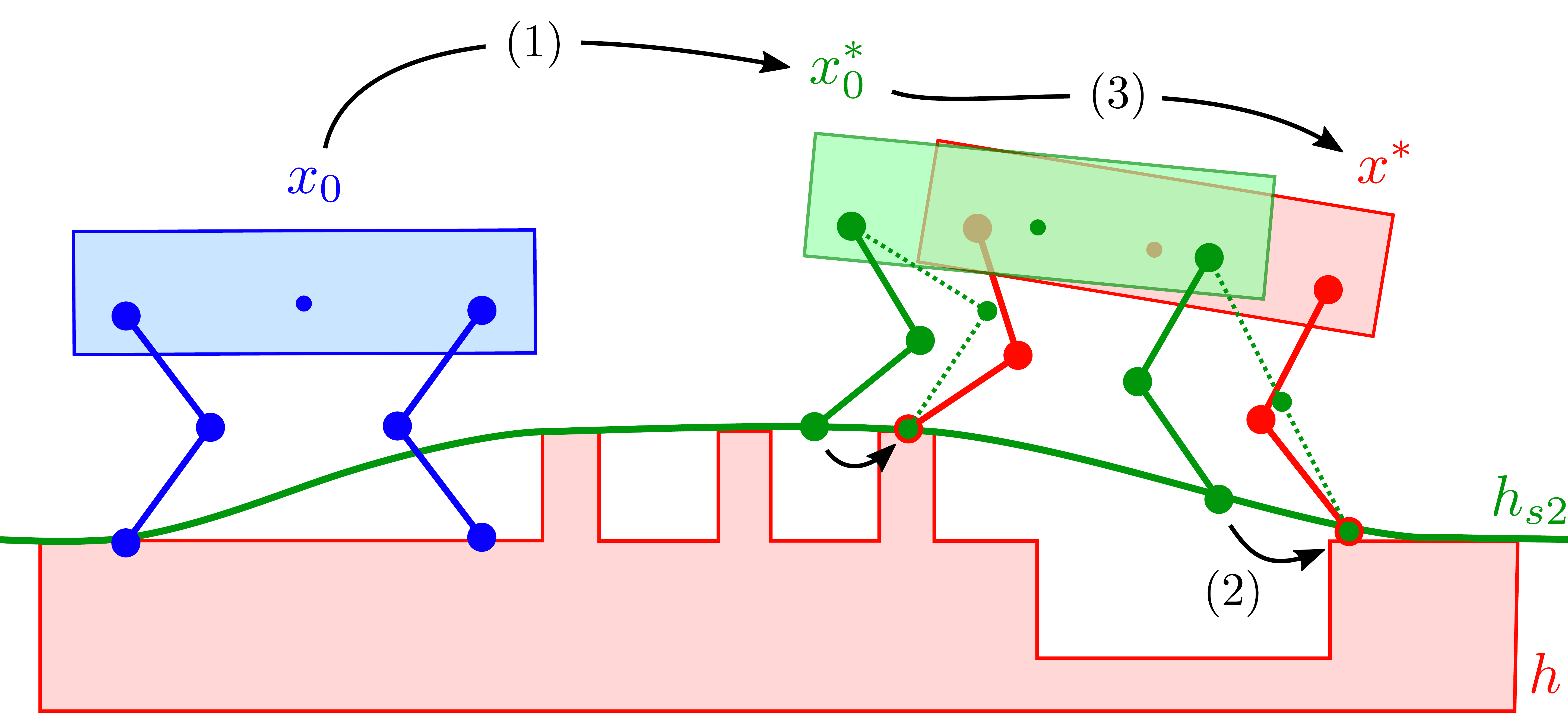}
\caption{(1) Given an initial guess $\vec x_0$ \highlight{(blue)} the motion is optimized over the virtual floor $h_{s2}$, resulting in the intermediate solution $\vec x_0^*$ (green). (2) The post-processed intermediate solution (dashed) is obtained by refining the footholds w.r.t. to the height map $h$. (3) The refined solution is used to warm start the original \ac{NLP}, constraining the motion to $h$ and producing the final solution $\vec x^*$ \highlight{(red)}.}
\label{fig:planning:solver_overview}
\end{figure}

We notice that task \eqref{ss:planning:base_pose_alignment} keeps the base pose aligned with $h_{s2}$. We thus expect the base trajectory to stay approximately constant across different iterations of the graduating optimization scheme. This leads to the idea of truncating the sequence after solving the first \ac{TO} problem while replacing the remaining sequence with a foothold refinement step. We do so by leveraging a simplified version of the batch search. The embedded cost functions penalize occluded cells of $\tilde h$, gradient and positive curvature of $h_{s1}$ and vertical distance to $h_{s2}$ within a \unit[0.4]{m} wide search area centered about the previously optimized foothold.

In the last stage, the optimized and refined state is subsequently passed over the actual \ac{NLP} as a new initial guess. The resulting three-staged optimization hierarchy is visualized in Fig.~\ref{fig:planning:solver_overview} and applied at each leg touch-down. 

\subsection{Swing Trajectory}
A collision-free swing trajectory $\vec s(t)$ is established that connects the lift-off position $\vec p_\text{lo}$ with the desired foothold $\vec p_i$. The gradient-free  method  is  geometrically motivated in Fig.~\ref{fig:planning:swing1}. 
The trajectory consists of two quintic splines, smoothly connected at midswing phase, corresponding to spatial location $\vec s_{0.5}$. The parametrization introduces $6\cdot 2$ spline coefficients, from which $3\cdot 3$ are given by the initial, final, and junction conditions. We assume that the edge of an obstacle is located in between the lift-off and touch-down position, along the step normal $\vec n_s$. We place the spline junction onto that edge, i.e., $\vec s_{0.5} = h_a \vec n_s + 0.5(\vec p_\text{lo} + \vec p_i)$, leaving one \ac{DOF} left. The apex height $h_a$ is found iteratively by lifting $\vec s_{0.5}$ along $\vec n_s$ until the obstacle is cleared. 
\begin{figure}
\begin{minipage}{0.55\columnwidth}
\centering
\subfloat[]{\label{main:a}\includegraphics[width=1\textwidth]{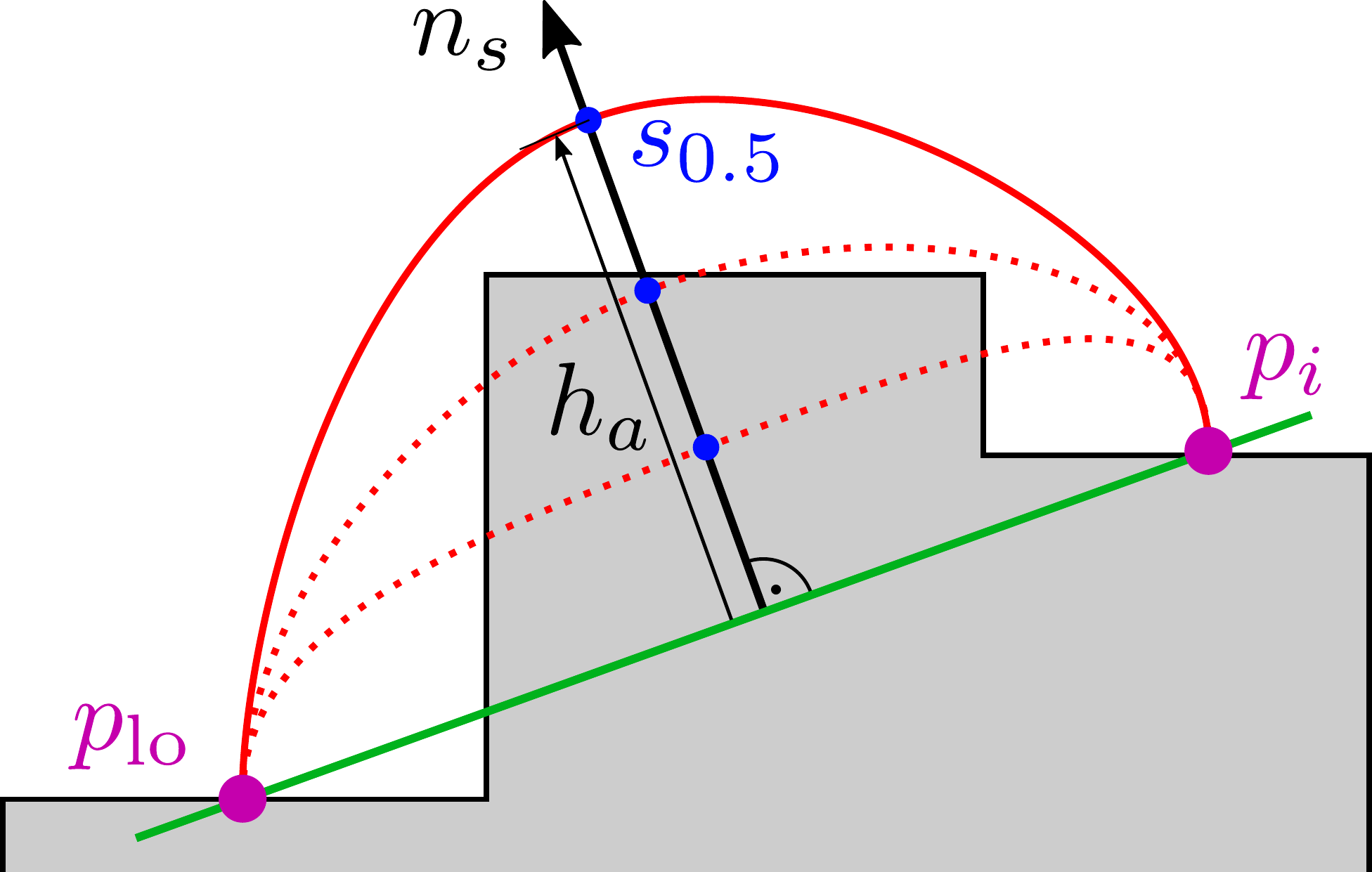} \label{fig:planning:swing1}} \\
\subfloat[]{\label{main:b}\includegraphics[width=1\textwidth]{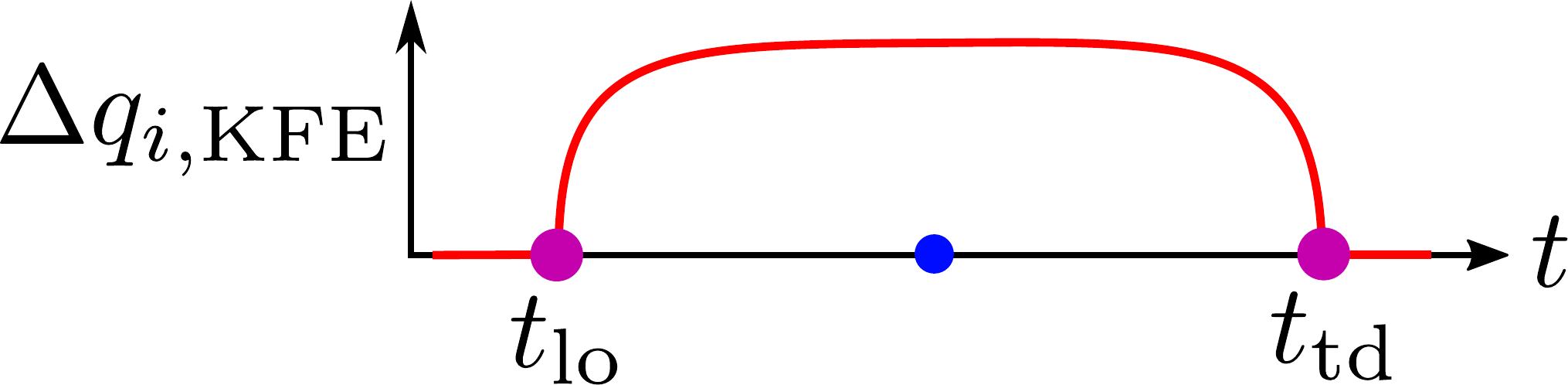} \label{fig:planning:swing2}}
\end{minipage}
\hspace{0.01\columnwidth}
\begin{minipage}{0.35\columnwidth}
\centering
\subfloat[]{\label{main:c}\includegraphics[width=1\textwidth]{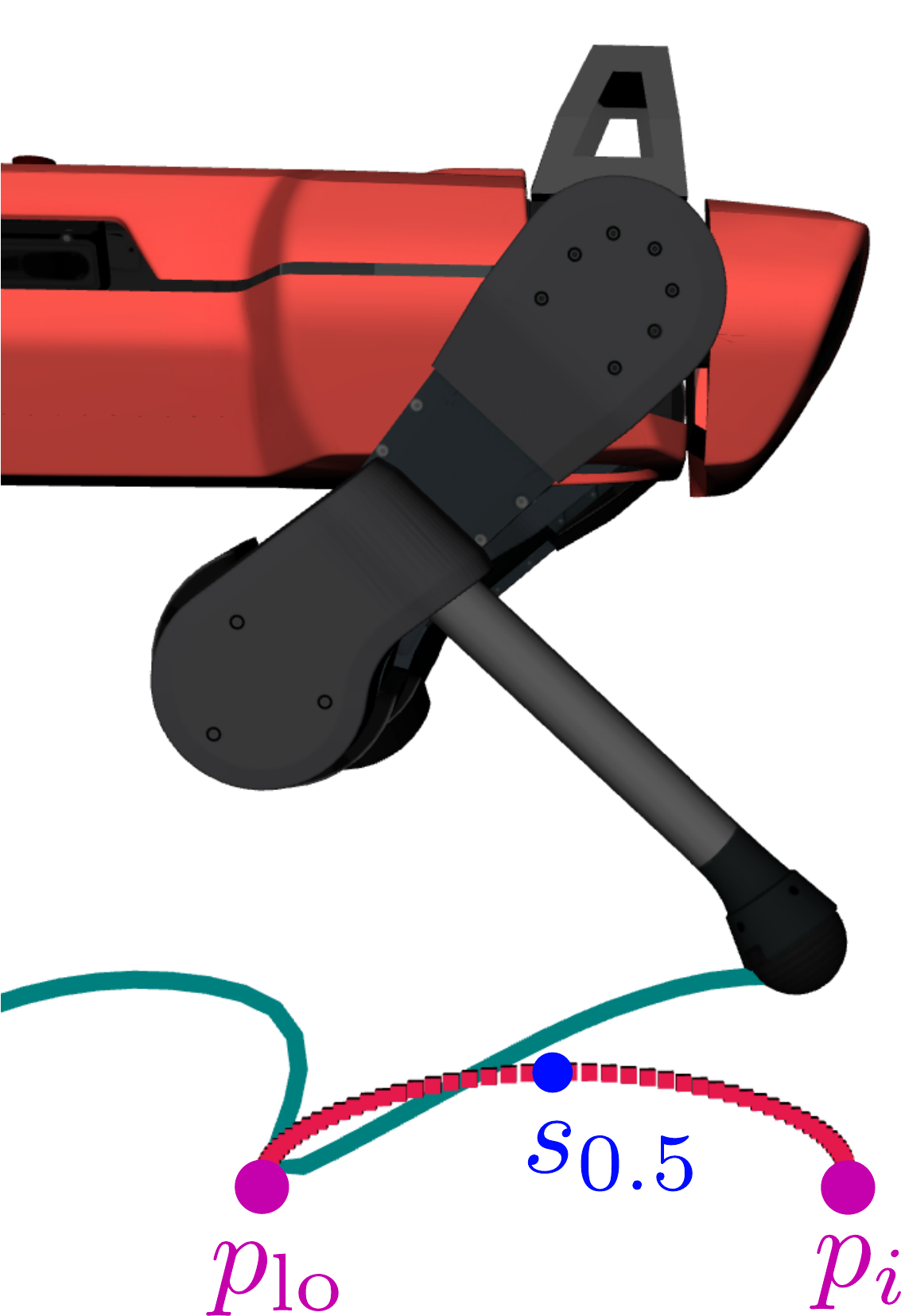} \label{fig:planning:swing3}}
\end{minipage}
\caption{(a) Geometric construction of the swing trajectroy in task space. (b) The knee trajectory flexes the joint between lift-off $t_\text{lo}$ and touch-down time $t_\text{td}$. (c) Overlaid task space trajectory.}
\label{fig:planning:swing}
\end{figure}

Due to tracking errors and odometry drift, a swing leg may still collide with the environment. After the collision, the foot travels along the vertical direction of the obstacle until a stable contact is detected. The new touch-down location shrinks the volume of the planned \ac{GIAC}, and stability of the truncated motion is no more guaranteed. We wish the collision to take place in an early stage of the swing phase, s.t. the \ac{GIA} vector is still contained in the stability margin. In other words, we wish the foot to travel in the shortest possible time to the apex. Setting up such trajectories in task space might require accelerating a lot of inertia.\footnote{The knee joint of ANYMal is relatively heavy} Instead, we overly the swing trajectory with a joint space offset trajectory, that only flexes joints carrying the smallest inertia, i.e.,  all \ac{KFE} joints. Let  $\vec q_i$ be the generalized coordinates of limb $i$, then the new swing motion becomes
\begin{equation}
\begin{aligned}
     \vec q_{i,\text{des,new}}(t) &= \vec q_{i,\text{des}}(t) + \Delta \vec q_i(t) \\
     \vec {\dot q}_{i,\text{des,new}}(t) &= \vec {\dot q}_{i,\text{des}}(t) + \Delta \vec {\dot q}_i(t) \\
     \vec {\ddot q}_{i,\text{des,new}}(t) &= \vec {\ddot q}_{i,\text{des}}(t) + \Delta \vec {\ddot q}_i(t).
\end{aligned}
\end{equation}
A valid offset trajectory $\Delta \vec q_i$ needs to satisfy
\begin{equation}
\begin{aligned}
    &j \neq \text{\ac{KFE}: } &\Delta \vec q_{ij}(t) &= 0 & \forall t& \\
    &j = \text{\ac{KFE}: }    &\Delta \vec q_{ij}(t) &= 0 & t_\text{td} < t < t_\text{lo}&.
\end{aligned}
\end{equation}
Fig.~\ref{fig:planning:swing2} shows a plot of the \ac{KFE} element. We first convert the swing trajectory into joint space using analytical inverse kinematics, pseudo-inverse differential kinematics and dynamics, yielding desired joint positions $\vec q_{i,\text{des}}(t)$, velocities $\vec {\dot q}_{i,\text{des}}(t)$ and accelerations $\vec {\ddot q}_{i,\text{des}}(t)$. After adding the offsets $\{\Delta \vec q_i, \Delta \dot{\vec q}_i, \Delta \ddot{\vec q}_i\}$ the signals are converted back to task space by solving the forward problem, leading to a new swing trajectory as shown in Fig.~\ref{fig:planning:swing3}.

\section{Whole Body Control} \label{s:wbc}
\subsection{Tracking} \label{ss:wbc:tracking}
\begin{table}
\centering
\caption{Task priorities used for \ac{WBC}.}
\begin{tabular}{l c c}
\toprule
task                                    & type      & priority  \\ \midrule
equations of motion                     & $=$       & $0$       \\
joint torque limits                     & $\leq$    & $0$       \\  \midrule
kinematic limits                        & $\leq$    & $1$       \\
friction pyramid (stance legs)          & $\leq$    & $1$       \\
no contact motion (stance legs)         & $=$       & $1$       \\
tracking in task space (swing legs)     & $=$       & $2$       \\
tracking in task space (torso)          & $=$       & $3$       \\ \midrule
tracking in joint space                 & $=$       &$4$        \\
minimize contact forces                 & $=$       &$4$        \\ \bottomrule
\end{tabular}
\label{tab:wbc_priorities}
\end{table}
The reference signals are tracked by a \ac{WBC}~\cite{BellicosoJenelten2017} in task space, while an impedance control law~\cite{fabianje2020} improves tracking under low loads. Table~\ref{tab:wbc_priorities} outlines the task priorities. The tasks are categorized into three blocks, where the first one contains physical constraints at highest priority. The middle block ensures stable contacts and tracks the reference trajectories in task space. The prioritization is such that a saturation of kinematic limits causes a base pose adaptation in favor of swing leg tracking. This feature is particularly useful in the case of regaining~\cite{Bellicoso2016}, i.e., an event where a leg expects a contact, but it has not been established yet. In this case, the foot is pushing vertically towards the ground with a constant velocity. The motion optimizer will adjust the pose trajectory accordingly, but with a delay of one optimization duration. In this short time period, the \ac{WBC} sacrifices tracking performance of the base pose to avoid a regaining leg to enter a singular configuration. We use the kinematics limits task introduced in~\cite{Bellicoso2016}, where we constrain only the \ac{KFE} joints to enforce the X-configuration.\footnote{Depending on the joint index, the angle is upper or lower bounded by $0$.}
The last block consists of low-priority tasks used to completely resolve the null-space.

\subsection{GM-Observer and Disturbance Rejection} \label{ss:wbc:gm_observer}
Consider the \ac{EOM} of an articulated robot structure with $n$ actuated joints
\begin{equation}
    \mat M(\vec q) \vec{\dot u} + \vec h(\vec q, \vec u) = \mat S^T \vec \tau + \mat J(\vec q)^T \vec F_\text{ext},
\end{equation}
with $\vec q\in\mathbb{R}^{(n+6)}$, $\vec u\in\mathbb{R}^{(n+6)}$ the generalized coordinates and velocities, $\vec \tau\in\mathbb{R}^n$ the joint torques, $\vec F_\text{ext} = \Mx{\vec f_B & \vec \tau_B & \vec f_1 & \ldots & \vec f_N} \in\mathbb{R}^{6+3N}$ a vector stacking external base force, base momentum and contact forces, $\mat M(\vec q)$ the mass matrix, $\vec h(\vec q, \vec u)$ the Coriolis, centrifugal and gravity terms, $\mat S^T$ the selection matrix of the actuated \acp{DOF} and $\mat J(\vec q) \in \mathbb{R}^{(n+6)\times (n+6)}$ the stacked contact Jacobians.
A \ac{GM} observer is a momentum integrating structure that estimates the external forces using only the measurements $\{\vec q, \vec u, \vec \tau\}$
\begin{equation} \label{eq:wbc:tau_ext}
    \vec {\hat \tau}_\text{ext} = \text{GM}(\vec q, \vec u, \vec \tau)  = \mat J(\vec q)^T \vec {\hat F}_\text{ext} \in \mathbb{R}^{n+6}.
\end{equation}

The filter does not depend on measured joint accelerations or mass matrix inversion and serves as a virtual sensor for ``external joint torques'' acting on the base. It is well studied~\cite{GMObserverHaddadin} and widely applied in locomotion for disturbance compensation~\cite{focchi2020heuristic} and contact estimation~\cite{GMObserverBledt}.

Instead of discretizing the continuous filter dynamics, we implement a discrete-time version proposed in~\cite{GMObserverBledt} with a cut-off frequency of \unit[40]{Hz}. 

The external torque estimate is post-processed using saturation, median, and low pass elements as illustrated in Fig.~\ref{fig:wbc:gm_observer}.
\begin{figure}
\centering
\includegraphics[width=1.0\columnwidth]{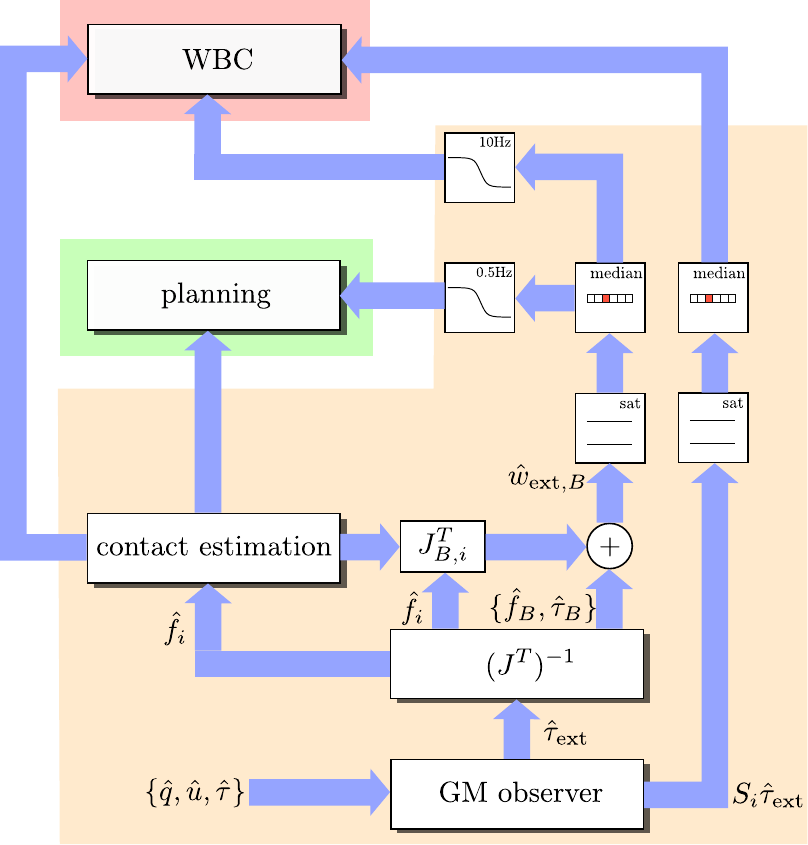}
\caption{The \ac{GM} observer estimates disturbances in joint space. The task space disturbance can be calculated by inverting the contact Jacobian $\mat J^T$. The vertical components of the estimated contact forces are thresholded to contact states. If a leg is swinging according to this estimate, contact forces are projected back into the floating base through the Jacobian matrix $\mat J_{B,i}^T$. The base wrench and joint torques are further mean and low-pass filtered to allow for smooth compensation on the planner and tracking side.}
\label{fig:wbc:gm_observer}
\end{figure}

A disturbance on a swing leg will cause an equivalent reaction on the base. It is desirable to project these disturbances back into the floating base. Let $\mat S_B$ be a matrix selecting the floating base \acp{DOF}, $\mat S_i$ a matrix selecting the actuated \acp{DOF} of the $i$th limb, and $\mat J_{B,i}(\vec q)$ be first 6 columns of the contact Jacobian associated to the $ith$ limb. Then, the base wrench disturbance becomes
\begin{equation}
    \vec {\hat w}_{\text{ext},B} = \mat S_B \vec {\hat F}_\text{ext} +  
    \sum_{i=\text{swing}}\mat J_{B,i}(\vec q)^T \mat S_i \vec {\hat F}_\text{ext}.
\end{equation}

The output of the \ac{GM} observer can be directly used to compensate for the external disturbances. A convenient way to do so is by manipulating the nonlinear term $\vec h(\vec q, \vec u)$
\begin{equation}
\begin{aligned}
     \mat S_B \vec {\tilde h}(\vec q, \vec u) &= \mat S_B \vec h(\vec q, \vec u)- \vec {\hat w}_{\text{ext},B}  \\
     \mat S_i \vec {\tilde h}(\vec q, \vec u) &= \mat S_i \vec h(\vec q, \vec u) - \mat S_i \vec{\hat \tau}_\text{ext} \quad \forall \text{ swinging } i.
\end{aligned}
\end{equation}
Replacing $\vec h$ with $\vec {\tilde h}$ in all tasks of the \ac{WBC} allows compensating directly for disturbances at the base and the swing feet.

\section{Results} \label{s:results}
We evaluate the performance of our perceptive control pipeline using the ANYmal platform. ANYmal is a fully ruggedized quadrupedal robot with $12$ actuated \acp{DOF}, designed to autonomously operate in challenging environments. Two LiDARs (Robo-Sense bperl) are mounted in the front and back of the torso. Elevation mapping runs at \unit[20]{Hz} on an onboard GPU (Jetson AGX Xavier) while control and state estimation is updated on a separate onboard CPU  (Intel  i7-8850H,2.6 GHz,  Hexa-core  64-bit) at \unit[400]{Hz}. \ac{TAMOLS} runs asynchronously at the maximum possible rate and optimizes trajectories for a prediction horizon of one full gait cycle.

\subsection{Computation Times (real world)}
\begin{figure}
\centering
\includegraphics[width=1.0\columnwidth]{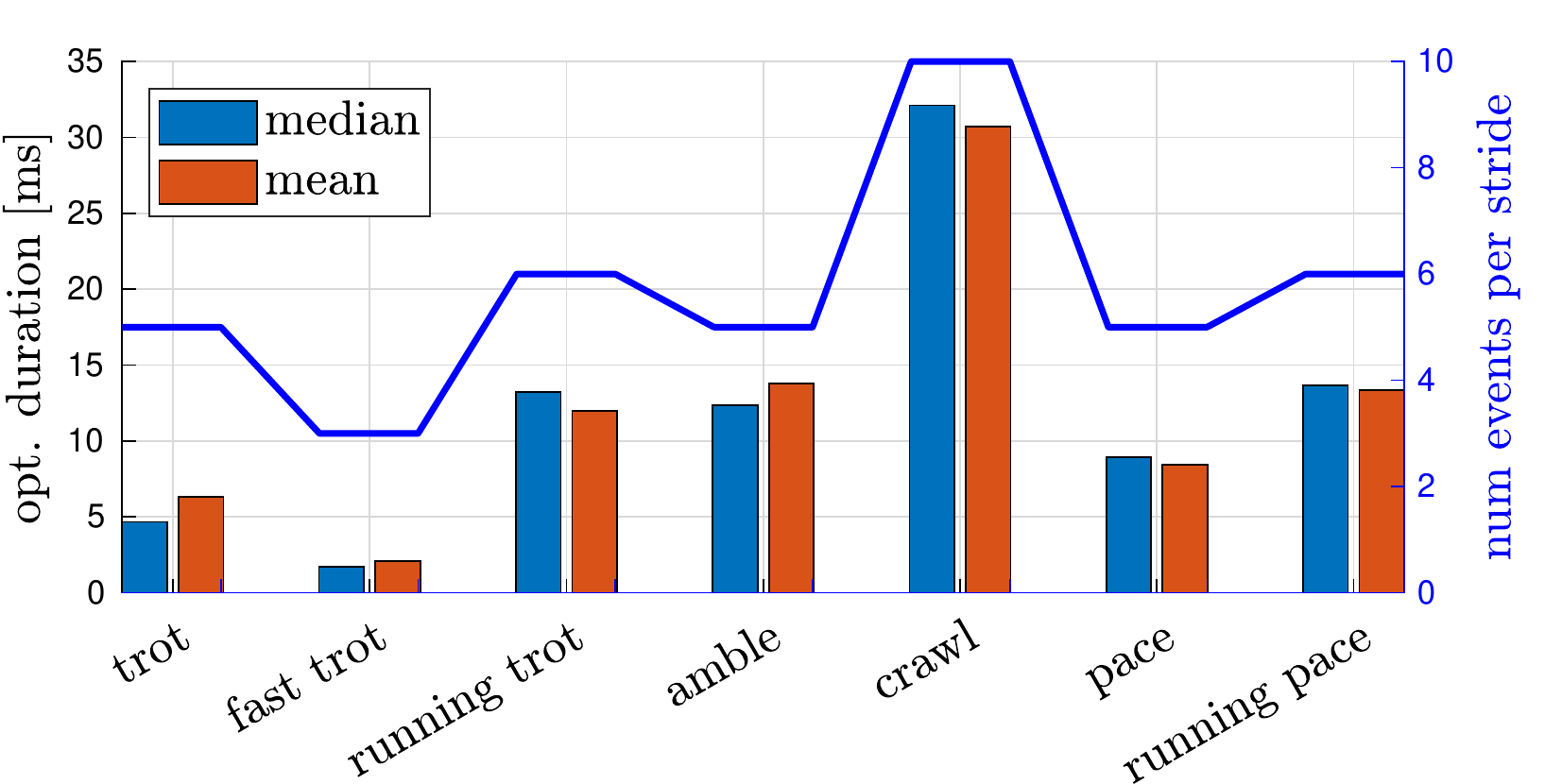}
\caption{Optimization duration (elapsed time between initialization to retrieving the solution) for a prediction horizon of one stride. Fast trot eliminates the full stance phase of trot, while the swing phases overlap for running trot. Amble is obtained from crawl by overlapping all swing phases until the triple stance phase vanishes. Pace/running pace has the same gait timing as trot/running trot but has a lateral leg pair alternating.}
\label{fig:results:optimization_durations}
\end{figure}
\ac{TAMOLS} is able to produce motions for a wide variety of dynamic gaits.
We record optimization durations while ANYmal is walking on flat ground with a maximum heading velocity. Mean and median, computed over five strides, are plotted in Fig.~\ref{fig:results:optimization_durations}. Each gait introduces a different number of decision variables, which grows linearly with the number of phases (splines) covered by the prediction horizon. This dependency can also be observed in the computation time.

For the most common gait trot, we achieve an average optimization duration of \unit[6.3]{ms}, which is 48 times faster than the latest state-of-the-art fully perceptive \ac{CMO} pipeline~\cite{Melon2021}. Compared to the blind \ac{MMO} controller presented in~\cite{Bellicoso2}, which shares the same parametrization of the base pose, we are only 2.9 times slower. If we also eliminate the full stance phase (fast trot), the average computation duration drops to \unit[2.1]{ms}.

Fig.~\ref{fig:results:optimization_duration_trot} contains the histogram for the trotting gait. The \ac{SQP} converges most often after one (\unit[5]{ms}) or two (\unit[10]{ms}) iterations. The peak values correspond to touch-down events, for which \ac{TAMOLS} also optimizes over the initial guess and performs a batch search. The latter takes in average about \unit[0.15]{ms}. On rough terrain, the peaks are slightly larger (about one SQP iteration) due to a mismatch between the height maps \highlight{$h_{s2}$ and $h$}, and thus the optimized initial guess and the solution.

\begin{figure}
\centering
\includegraphics[width=1.0\columnwidth]{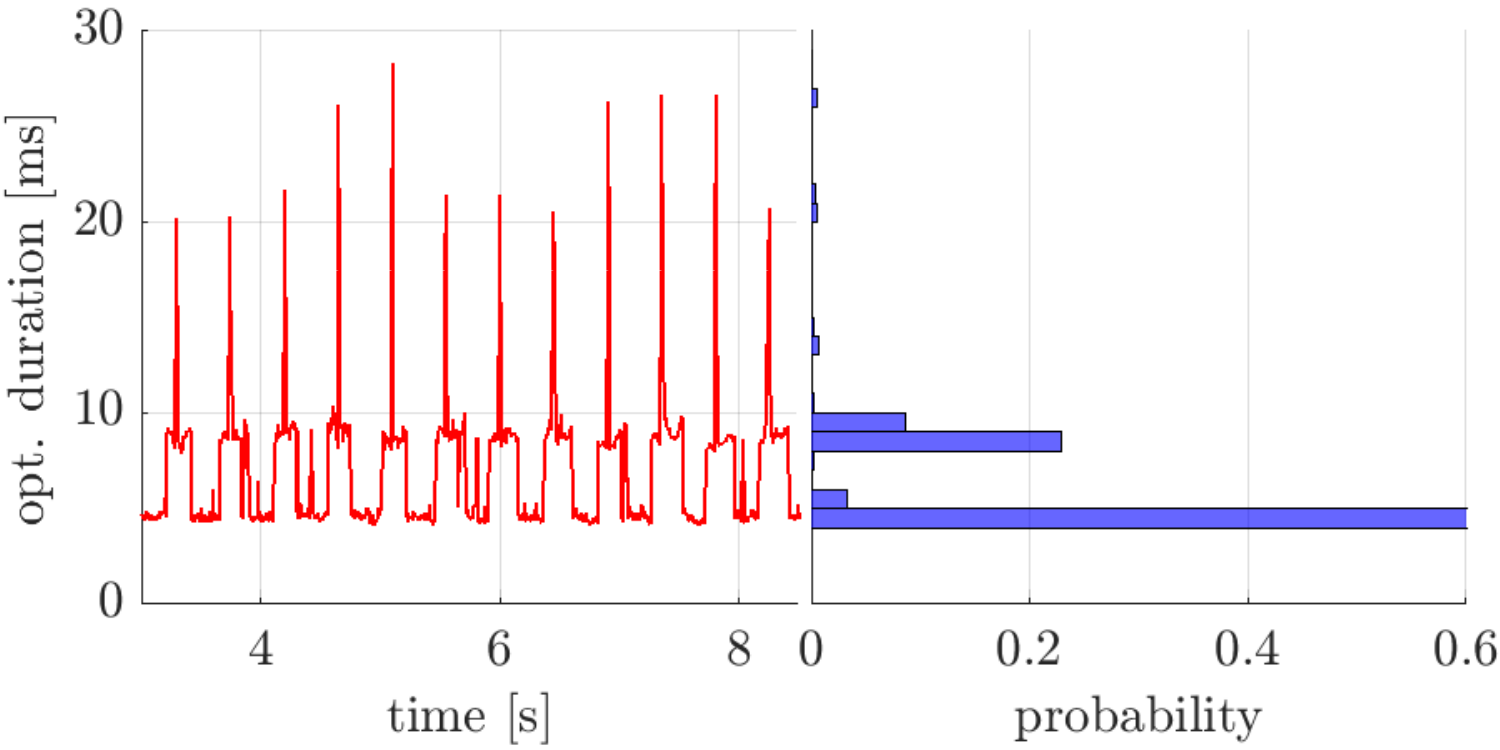}
\caption{Left: Computation duration recorded over time for a trotting gait. Right: Corresponding normalized histogram.}
\label{fig:results:optimization_duration_trot}
\end{figure}

\subsection{Staircase (simulation)}
In our previous work~\cite{fabianje2020} we have presented a \ac{MMO} pipeline and discussed its limitations. We found that a significant amount of failure cases on stairs were caused by knee joint collisions with edges (upstairs) or by violating reachability constraints (downstairs). The experimental set-up included a simulated staircase of $12$ treads, which was passed $18$ times in ascending and descending direction using a trotting gait. The success rates are recapitulated in table~\ref{tab:results:benchmark_dyn_gaits}.

We repeat the exact same experiment with \ac{TAMOLS}. 
The different limb lengths of the robot versions are taken into account by increasing height and width of each tread by a factor 1.27.\footnote{The thigh length increased from \unit[0.25]{m} to \unit[0.3]{m} and the shin length from \unit[0.3]{m} to \unit[0.32]{m}. The total relative limb extension is \unit[12.7]{\%} longer.} We further choose a reference velocity that allows the robot to progress one tread per step.

The statistical results are presented in table~\ref{tab:results:benchmark_dyn_gaits}. While trotting upstairs, the \ac{KFE} joints occasionally collide with the edges, but the robot is always able to keep its balance. Contrary to our previous control framework, we do not encounter issues with leg over-extensions. This can be explained by the kinematic constraints that are present in the prediction and the tracking level.
Moreover, we successfully repeat the experiment with three additional dynamic gaits: running trot, amble, and pace. For the sake of completeness, the experiment is also performed for crawl. Due to its large stride, the robot is very sensitive to contact state mismatches (downstairs) and fails to react to knee joint collisions (upstairs).

\begin{table}
\centering
\caption{Success rate scored while waling on stairs.}
\begin{tabular}{l c c c}
\toprule
method          & batch search  & \multicolumn{2}{c}{\ac{TAMOLS}}      \\
gait            & trot              & trot, amble, running trot, pace  & crawl         \\ \cmidrule(lr){2-2} \cmidrule(lr){3-4}
up              & $10/18$           & $18/18$       & $16/18$    \\ 
down            & $14/18$           & $18/18$       & $14/18$   \\ \bottomrule
\end{tabular}
\label{tab:results:benchmark_dyn_gaits}
\end{table}

\subsection{Staircase (real world)}
\begin{figure*}
\centering
\includegraphics[width=1.0\textwidth]{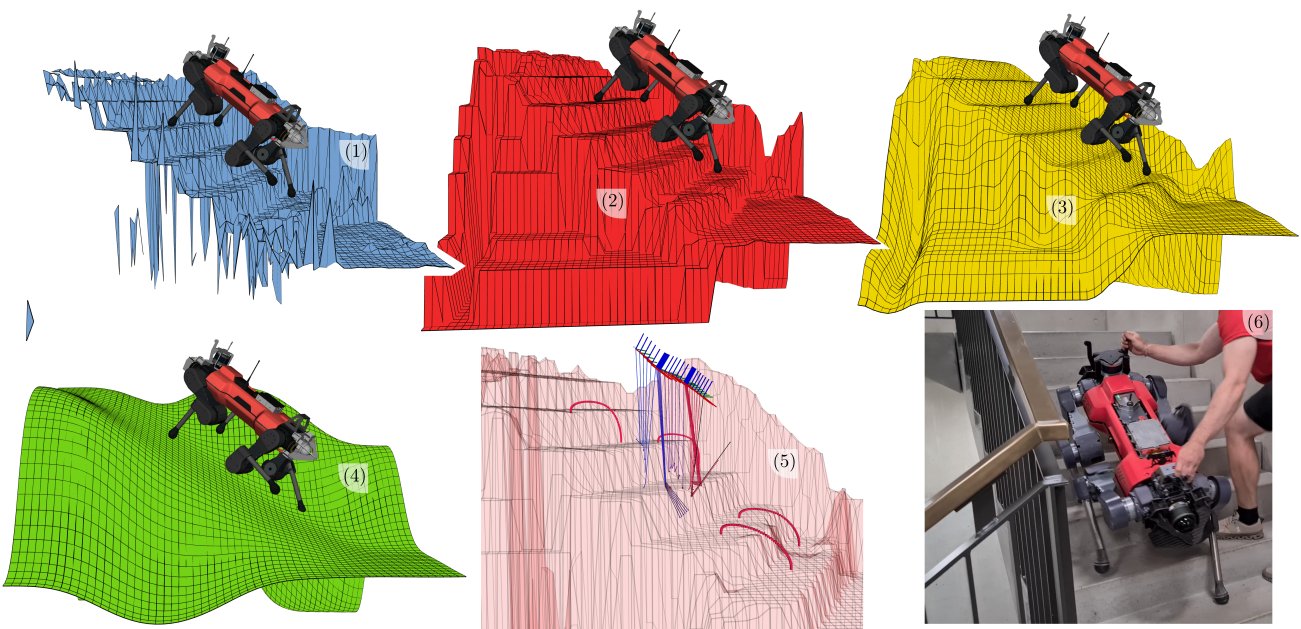}
\caption{ANYmal trotting up stairs. 1) Visualization of the robot together with the raw elevation map $\tilde h$. 2) The in-painted and de-noised elevation map $h$ is used for height constraints on the footholds. 3) The Gaussian-smoothed height map $h_{s1}$ is used for edge avoidance.  4) The robot aligns its torso with the virtual floor map $h_{s2}$. 5) Visualization of the motion plan in task space. 6) Snap-shot of the experiment.}
\label{fig:results:stairs_exp}
\end{figure*}
We take ANYmal to the real world and walk upstairs using a trot in the building of our lab. Two floors are connected with each other by $20$ treads where each platform has the dimensions $\unit[29]{cm} \times  \unit[17]{cm}$, forming an inclination of \unit[36]{deg}. The operator commands a heading speed of \unit[0.45]{m/s} whereas the realized average velocity is \unit[0.37]{m/s}. The tracking error originates mainly from the reduced feasible space imposed by the stair geometry: The robot prefers to place the footholds, s.t. all feet either clear zero, one, or two treads per stride. Depending on the magnitude of the reference velocity, a different optimum is generated, which was found at one tread per step for our experiment. This behavior is further investigated in the subsequent experiment~\ref{ss:results:vel_mod}.

In Fig. \ref{fig:results:stairs_exp} we visualize the three height maps used for motion optimization along with the raw heightmap. For a staircase, the computation of the virtual floor $h_{s2}$ equals Gaussian filtering of the heightmap $h$. The smooth map $h_{s1}$ attains zero gradient in the middle of the tread, thereby pushing footholds away from edges towards the center line.

\subsection{Velocity Modulation (Simulation)} \label{ss:results:vel_mod}
We want to investigate the relation between commanded and realized velocity when walking 12 steps upstairs. We use the stair parameters from the previous experiment and a fast trot. The relationship between reference velocity and the number of stair treads traversed per step is visualized in Fig.~\ref{fig:results:velocity_modulation}. For heading reference speeds larger than \unit[0.9]{m/s}, the robot always clears two treads per step. Kinematic limits prevent the robot from progressing any faster. If the commanded velocity is reduced below \unit[0.9]{m/s}, the robot alternates between clearing two and one treads per step. Clearing exactly one tread per step is found as the optimal solution over a wide range of reference velocities $\unit[\{0.45,\ldots, 0.55\}]{m/s}$. 
\begin{figure}
\centering
\includegraphics[width=1.0\columnwidth]{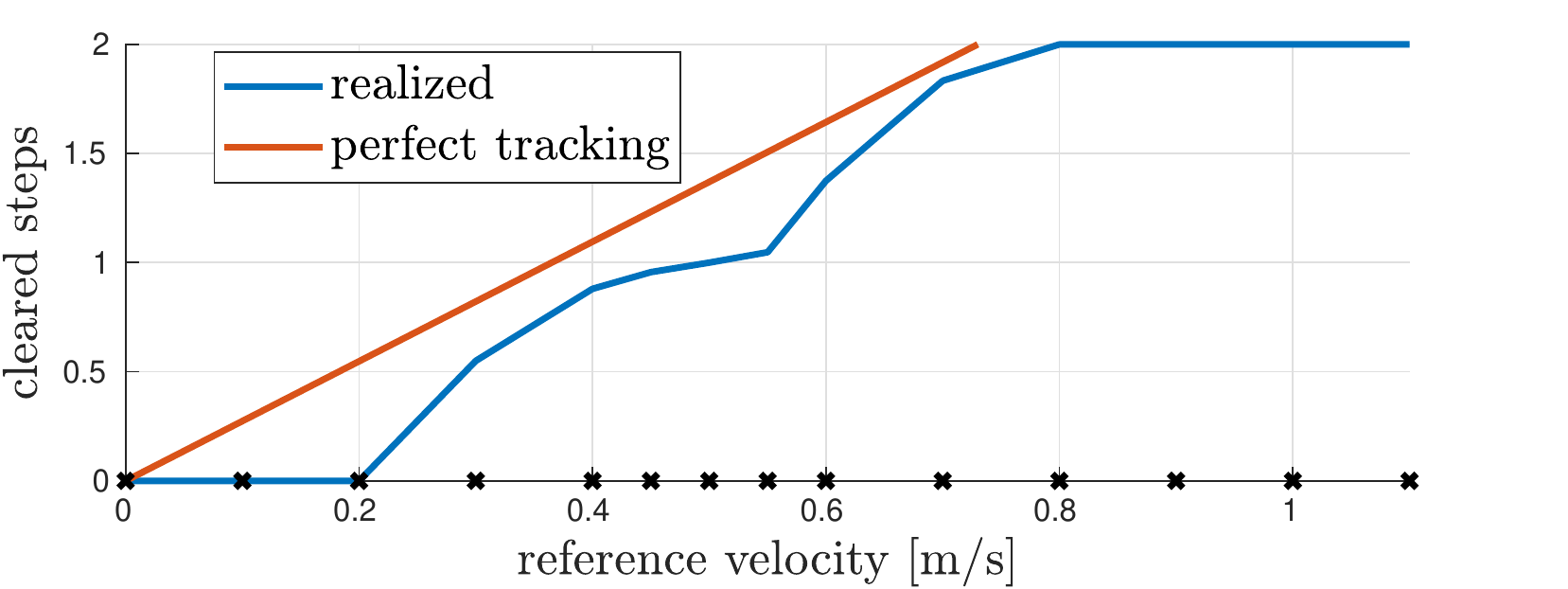}
\caption{Average number of stair treads cleared per step and per foot, which is proportional to the realized average velocity. The count starts when the second foot strikes the first tread and stops if the first foot reaches the top platform. Blue: Results for a fast trot. The attractive velocities clear zero, one, or two treads per step. Red: Hypothetical curve that would lead to perfect velocity tracking. Black: Reference velocities for which experiments were conducted.}
\label{fig:results:velocity_modulation}
\end{figure}

It can be observed that the realized velocity is always smaller than the reference. TAMOLS trades off robustness against velocity tracking. Increasing the weight for minimizing $\varepsilon$ in \eqref{eq:planning:dyn} increases the robustness margin but also degrades velocity tracking.

\subsection{Gap (real world)}
The strength of the proposed \ac{TO} method lies in the generalization to stairs, gaps, and stepping stones.  
In the experiment outlined in Fig.~\ref{fig:results:gaps_exp}, we test the controller's behavior in the presence of a gap. The robot is required to traverse over a pallet and slope using an ambling gait. Both obstacles are placed around \unit[30]{cm} apart from each other.

By penalizing inclinations, the robot successfully avoids the steep part of the pallet. 
The height layer $h_{s2}$ connects the pallet with the slope, thus creating an artificial floor. The cost function \highlight{of the batch search} increases quadratically with the vertical foothold distance to $h_{s2}$, virtually rendering the gap unattractive to step. The experiment was repeated five times, from which the robot never stepped on the inclined slope and placed once \ac{RH} foot into the gap. 
\begin{figure}
\centering
\includegraphics[width=1.0\columnwidth]{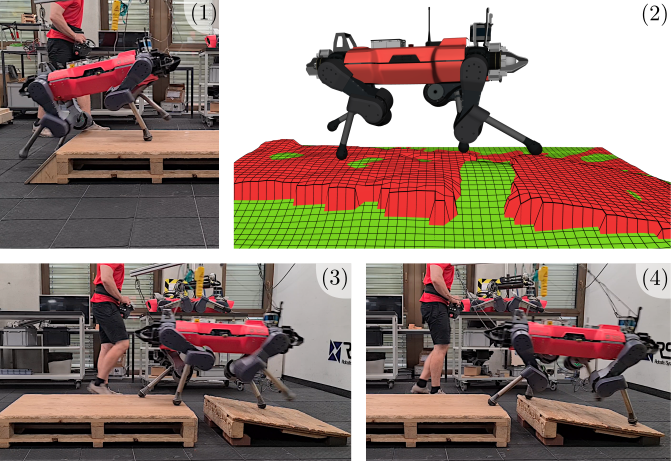}
\caption{ANYMal ambling over two obstacles, forming a gap. The commanded heading speed is \unit[0.7]{m/s}. The gap has a height of \unit[20]{cm} and width of \unit[27]{cm}. The second pictures shows the two layers $h$ (red) and $h_{s2}$ (green).}
\label{fig:results:gaps_exp}
\end{figure}

\subsection{Stepping Stones (real world)}
\begin{figure*}
\centering
\includegraphics[width=1.0\textwidth]{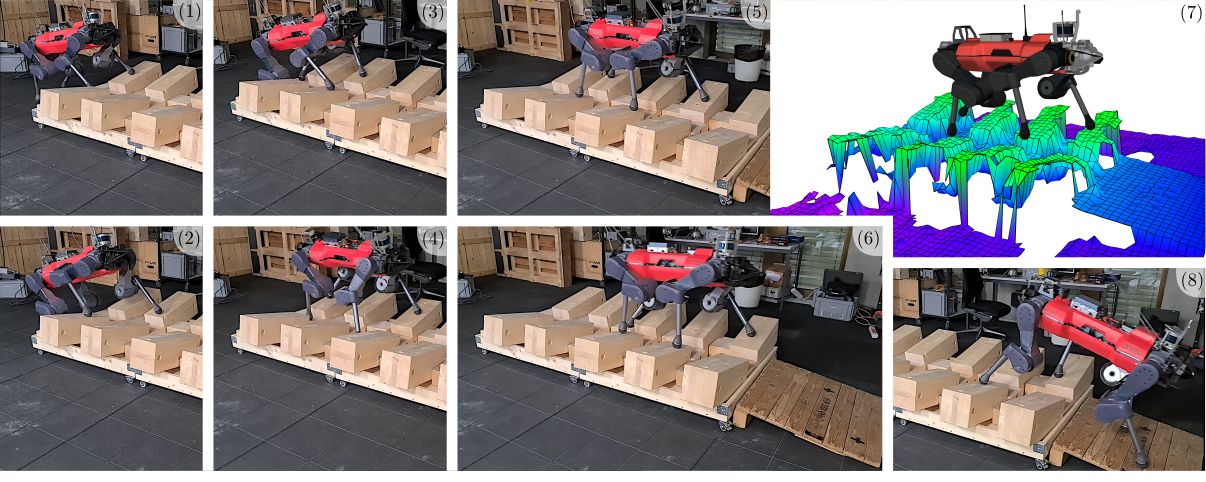}
\caption{Anymal traversing a stepping stone parkour made of inclined wooden bricks using a trotting gait. The commanded heading velocity is \unit[0.4]{m/s}. The dimensions of each brick are \unit[20]{cm}$\times$\unit[20]{cm}$\times$\unit[50]{cm} and the gap between any two adjacent bricks measures $\unit[20]{cm}$. Due to the inclined surfaces, the model assumptions are not perfectly satisfied in this experiment. Picture 7 shows the visualization of the robot along the raw elevation map.}
\label{fig:results:stepping_stones_exp}
\end{figure*}
We validate the control performance in a stepping stone experiment as described in Fig.~\ref{fig:results:stepping_stones_exp}. By penalizing gradients of the smooth map $h_{s1}$, the robot centers the footholds in the middle of the stepping stones. Due to odometry drift, the bricks, as seen by the elevation map, may be translated by a few centimeters from their actual locations. This explains why the robot sometimes steps close to the border.\footnote{When visualizing the robot together with the elevation map, the footholds appear in the middle of the brick.}

The robot is attached to an industrial crane during the experiment, exerting an external force on the base. The wrench estimate of the \ac{GM} observer is plotted in Fig.~\ref{fig:results:stepping_stones_gm_observer}. It can be seen that the forces do not converge to zero at rest ($t<\unit[2]{s}$ and $t>\unit[25]{s}$). This is due to modeling errors, in particular inaccurate whole body mass and \ac{COM} location. Each peak of the unfiltered forces corresponds to one step, which requires the robot to push the crane for the same distance walked. Towards the end of the path, the robot attempts to walk down the slope while the crane is holding back the torso. This disturbance is experienced majorly as a horizontal force and pitch moment, as can be seen in the plots for the time interval $t \in (\unit[20]{s}, \unit[25]{s})$.
\begin{figure}
\centering
\includegraphics[width=1.0\columnwidth]{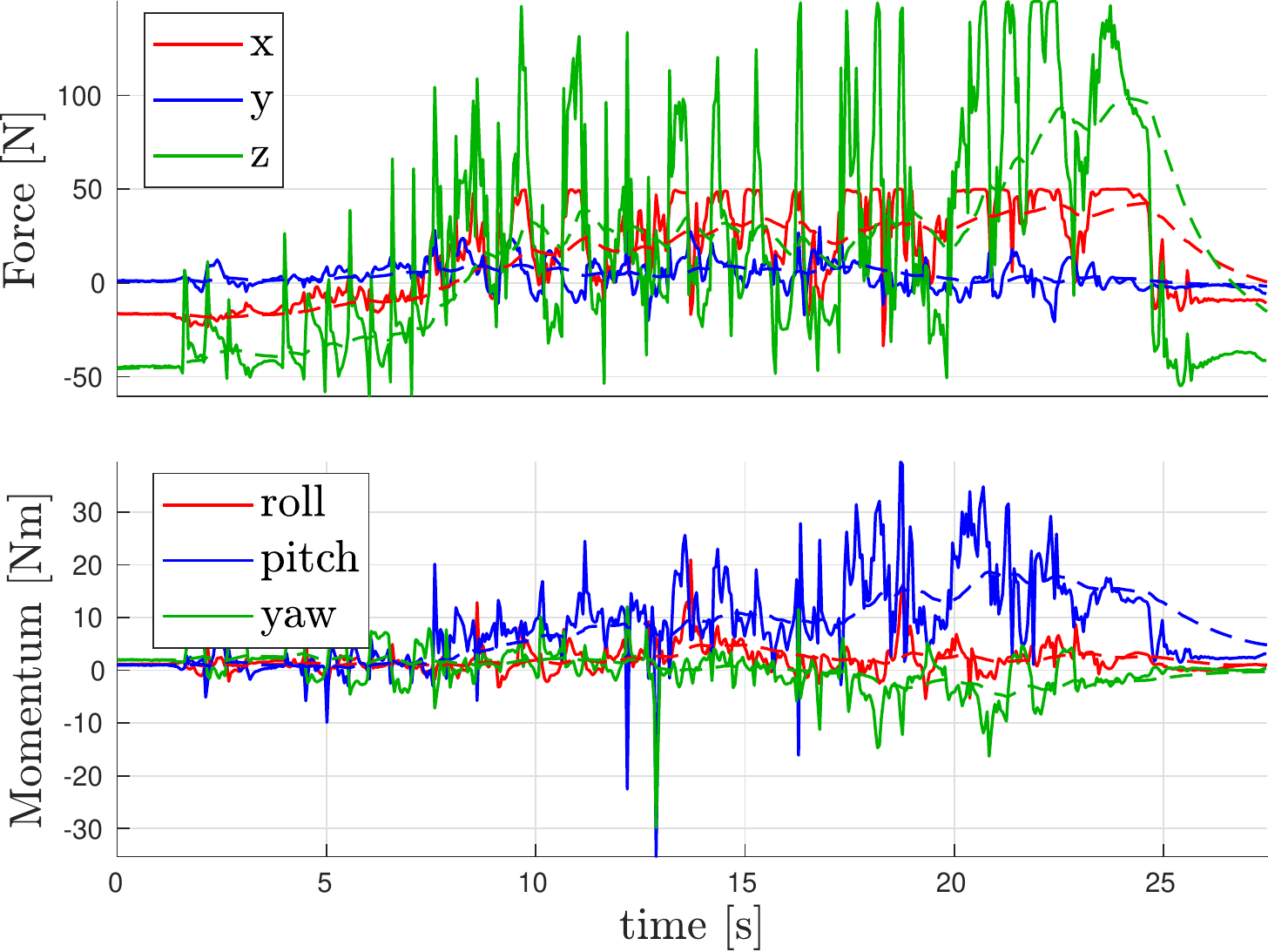}
\caption{Estimate of the external force (top) and moment (bottom) acting on the base. Solid: Low pass filtered at \unit[10.0]{Hz}, compensated at tracking level. Dashed: Low pass filtered at \unit[0.5]{Hz}, compensated at planning level. Horizontal/vertical forces are thresholded at \unit[50]{N},\unit[150]{N}.}
\label{fig:results:stepping_stones_gm_observer}
\end{figure}

\subsection{Limitations}
The major drawback of \ac{TAMOLS} originates from the dynamic model approximation: The method is only guaranteed to find feasible \highlight{(i.e., weak contact stable)} motions if it finds \highlight{horizontal} contact surfaces. If at least one of the feet is located on a tilted plane, the \ac{WBC} may deviate from the planned trajectory in order to ensure the no-slip condition. Theoretically, it is possible to generalize the model to a more general terrain structure by enforcing the last property in prop.~\ref{a:dyn_prop1} and~\ref{a:dyn_prop2} through hard constraints. However, it may be very difficult to find such footholds.

Second, \ac{TAMOLS} does not take into account the full kinematics. The simplified kinematic constraints formulated in task space greatly help to avoid leg over-extension but knee joint collisions still remain a problem.

As a final limitation, we mention the strong dependency of the footholds and base pose on the quality of the elevation map. State estimator drift shifts the map relative to odometry. This leads to a mismatch between realized and planned foot locations, which, in extreme cases, can destabilize the system.

\section{Conclusion} \label{s:conclusion}
First, we have derived a differentiable and contact-force free dynamic stability metric. It was shown to be at most as complex as \ac{ZMP} constraints while attaining a large fraction of validity of \ac{SRBD}.
In the second part of this work, we have embedded our model into a terrain-aware motion optimizer. Experiments have suggested that underlying assumptions generalize well to rough, human-engineered environments. By deploying graduated optimization, we have shown that it is possible to jointly optimize footholds and base pose over rough terrain in real-time. Compared to our previous work, we have shown that the inclusion of footholds into the optimization problem successfully removes the appearance of leg-over-extension.

%% file: chapters/appendix.tex
\section{}  \label{a:comparison}
\subsection{Geometric Derivation and Interpretation of the \acs{GIAC}} \label{ss:giac_interpretation}
\begin{figure}
\centering
\includegraphics[width=1.0\columnwidth]{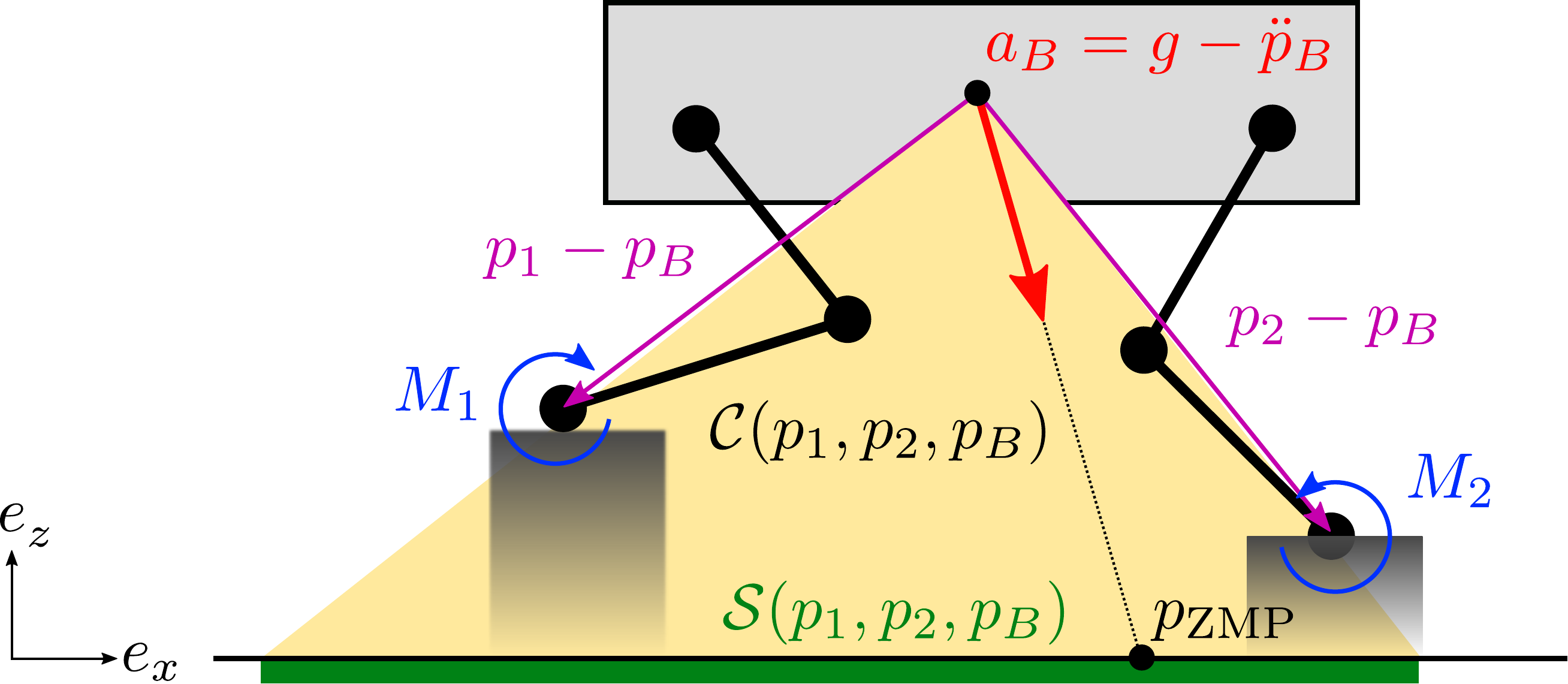}
\caption{The \ac{GIA} vector $\vec a_B = \vec g - \vec {\ddot p}_B$ induces the moment $\vec M_{1,2}$ about the contact positions $\vec p_{1,2}$. The cone $\mathcal{C}$, spanned by the base $\vec p_B$ and the two footholds, defines the set of all admissible $\vec a_B$. The \acs{ZMP} equals the intersection of the ray $\vec p_B + \gamma \vec a_B$ with the ground. The associated support polygon $\mathcal{S}$ is obtained by projecting the footholds onto the ground.}
\label{fig:models:dynamics_2d}
\end{figure}
To start with, let us consider a two-legged $2D$ robot that establishes point-contacts with the environment at $\vec p_1$ and $\vec p_2$, as illustrated in Fig.~\ref{fig:models:dynamics_2d}. For this particular example we assume
\begin{assumption}[Inverted pendulum mode] \label{assumption:models:pendulum}
1) The limbs have zero mass, 2) the base inertia is negligible, 3) contact forces can only push on the ground, 4) friction is infinitely large, and 5) contacts are established on horizontal planes.
\end{assumption}
The \acl{GIA} vector $\vec a_B$ induces a moment $\vec M_i = m(\vec p_B - \vec p_i) \times \vec a_B$ around the two feet.
The rays defined by the vectors $\vec p_i-\vec p_B$ play an important role as they define the support cone:
\begin{definition}[\acs{GIAC}]
The convex cone of rays, connecting the base position and the footholds, is called \acf{GIAC} $\mathcal{C}$.
\end{definition}

If the \ac{GIA} vector $\vec a_B$ stays on the left side of the vector $\vec p_B-\vec p_2$, the robot may be stable. However, if it were located on the opposite side, the induced moment would be such that the robot might tip over the right foot. Similar considerations also apply to the opposite foot, and it can be concluded that 
\begin{prop} \label{prop:models:zmp_linear}
Under assumption~\ref{assumption:models:pendulum}, a robot does not tip over any supporting foot if $\vec a_B \in \mathcal{C}$.
\end{prop}
We note that the opposite is not always true, i.e., $\vec a_B \notin \mathcal{C}$ does not imply that the robot will tip over a supporting foot. For instance, taking the example of Fig.~\ref{fig:models:dynamics_2d}, the contact force $\vec f_1$ of the left foot can create a counter moment $(\vec p_1-\vec p_B)\times \vec f_1$ to balance a \ac{GIA} vector that lies on the right side of the \ac{GIAC}.

First and foremost, notice that prop.~\ref{prop:models:zmp_linear} is not restricted to a certain dimensionality. In two dimensions, the stability criterion $\vec a_B \in \mathcal{C}$ can be easily converted to inequality constraints by checking the sign of the induced moment about each foot. A similar conversion can be found in three dimensions, where we consider the moment induced about the line connecting two neighboring footholds $\vec p_{ij} = \vec p_j - \vec p_i$, resulting in the condition $m\vec p_{ij}^T \cdot \big[(\vec p_B - \vec p_i) \times \vec a_B \big] \leq 0$, or, more compactly
\begin{equation} \label{eq:comparison_giac_simple}
    \det(\vec p_{ij}, \vec p_B - \vec p_i, \vec a_B) \leq 0.
\end{equation}
The formulation~\eqref{eq:comparison_giac_simple} is more conservative than prop.~\ref{prop:models:zmp_linear} because it restricts the \ac{GIA} vector to the largest inscribing convex cone $\tilde{\mathcal{C}}$ of $\mathcal{C}$. This cone satisfies $\mathcal{C}= \tilde{\mathcal{C}}$ (for $N\leq 3$) or $\mathcal{C}\subseteq \tilde{\mathcal{C}} \cup \emptyset$ (for $N>3$).
Fig.~\ref{fig:planning:convex_cone} visualizes three possible geometries of the inscribing cone for a full stance phase, i.e., $N = 4$.

By comparing~\eqref{eq:comparison_giac_simple} with \eqref{eq:models:m2}, we can now see that a non-zero angular momentum derivative can be observed as a change of shape and size of the \ac{GIAC}.

\subsection{Relation to ZMP} \label{a:comparison_zmp}
By definition, the \ac{ZMP} is a point $\vec p_\text{ZMP}(\vec O, \vec n)$ where the moment caused by the contact wrench aligns with the unit vector $\vec n$~\cite{CaronZMP}. If the moment is computed w.r.t. to a point $\vec O$, then $\vec p_\text{ZMP}(\vec O, \vec n)$ is located on the plane $\vec \Pi(\vec O, \vec n)$ which contains $\vec O$ and is orthogonal to $\vec n$. 
Without loss of generality, we can assume $\vec O = \vec O_W$ and $\vec n = \vec e_z$. For the time being, we omit the rate of change of the angular momentum, then~\cite{zmp}
\begin{equation} \label{eq:models:zmp}
    \vec p_\text{ZMP} \coloneqq \frac{\vec e_z \times (\vec p_B \times \vec a_B)}{\vec e_z^T \vec a_B}.
\end{equation}

A simple computation reveals that $\vec p_\text{ZMP}$ is the projection of the base $\vec p_B$ along the vector $\vec a_B$ onto the ground plane. We write this projection as $\mathcal{P}_{\vec a_B}(\vec p_B)$ with $\mathcal{P}_{\vec a_B}: \mathbb{R}^3 \to \vec \Pi(\vec O_W, \vec e_z)$.
\begin{proof}
The point~\eqref{eq:models:zmp} can be written as
\begin{equation}
    \vec p_\text{ZMP}
    = \frac{\vec p_B (\vec e_z^T \vec a_B) - \vec a_B (\vec e_z^T \vec p_B)}{\vec e_z^T \vec a_B}
    = \vec p_B - \vec a_B \frac{\vec e_z^T \vec p_B}{\vec e_z^T \vec a_B}.
\end{equation}
The intersection of the ray, extending $\vec p_B$ along $\vec a_B$, with the plane $\vec \Pi(\vec O_W, \vec e_z)$ satisfies the equation $\vec e_z^T (\vec p_B + \vec a_B \gamma ) = 0$. Solving for $\gamma$, the intersection can be found as
\begin{equation}
    \mathcal{P}_{\vec a_B}(\vec p_B) = \vec p_B + \vec a_B \gamma 
    = \vec p_B - \vec a_B \frac{\vec e_z^T \vec p_B}{\vec e_z^T\vec a_B}.
\end{equation}
\end{proof}

This relation amounts to the following proposition:
\begin{prop}
On flat ground, prop.~\ref{prop:models:zmp_linear} equals the \ac{ZMP} stability criterion, i.e., a motion of a robot is stable if $\vec p_\text{ZMP} \in \mathcal{S}$ with $\mathcal{S}$ being the support polygon.
\end{prop}
\begin{proof}
Prop.~\ref{prop:models:zmp_linear} can be written in span representation as
\begin{equation} \label{prop:models:zmp_linear_span}
    \vec p_B + \vec a_B = \sum_i \lambda_i (\vec p_i - \vec p_B), \quad \lambda_i \geq 0.
\end{equation}
The projection of the left-hand side of~\eqref{prop:models:zmp_linear_span} on the ground was shown to yield $\vec p_\text{ZMP}$. Because each point along the ray $\vec p_B + \vec a_B \gamma$ is contained in the \ac{GIAC}, and $\vec p_\text{ZMP}$ is part of that ray, we can write
\begin{equation} \label{prop:models:zmp_projection_span}
   \vec p_\text{ZMP} = \sum_i \tilde \lambda_i\mathcal{P}_{\vec p_i-\vec p_B}(\vec p_B), \quad   \lambda_i \neq \tilde \lambda_i \geq 0.
\end{equation}
On flat ground we have that $\mathcal{P}_{\vec p_i-\vec p_B}(\vec p_B) = \vec p_i$, and the set spanned by the right hand side of~\eqref{prop:models:zmp_projection_span} simplifies to the convex hull of contacts, conventionally termed \emph{support polygon}.
\end{proof}

The concept of \ac{ZMP} can be extended to also account for the angular momentum. The resulting model is the equivalent of an inverted pendulum with an attached flywheel and its \ac{ZMP} can computed as
\begin{equation} \label{a:zmp:angular_momentum}
\begin{aligned}
    \vec {\tilde p}_\text{ZMP} &= \frac{\vec e_z \times (\vec p_B \times m\vec a_B -  \vec{\dot L}_B)}{\vec me_z^T \vec a_B} \\
    &=
    \frac{\vec e_z \times (\vec p_B \times \vec a_B)}{\vec e_z^T \vec a_B} - 
   \frac{\vec e_z \times \vec{\dot L}_B}{m\vec e_z \vec a_B}.
\end{aligned}
\end{equation}
Hence, the additional angular momentum derivative leads to a shift of the \ac{ZMP} while leaving the support polygon unchanged. 

On flat ground, the \ac{ZMP} is typically computed w.r.t. the ground plane, while rough terrains require the definition of some \emph{virtual plane}~\cite{Sugihara2002, CaronZMP, BellicosoJenelten2017}. The notion of \ac{ZMP} can thus be easily extended to any terrain geometry; however, the generalization of the support polygon is not straightforward. Sugihara~\cite{Sugihara2002} introduced virtual contact locations as the intersection of the virtual plane with lines, connecting the \ac{COM} and the footholds. Harada~\cite{Harada2006} derived a projection from a convex 3D support volume to the virtual plane. Caron~\cite{CaronZMP} introduced a method for projecting the entire \ac{CWC}, thereby also accounting for pull contacts and frictional constraints. In contrast to those methods, the \ac{GIAC} model does not depend on rigorous, geometric construction of the support area. 

A different line of research has tried to simplify the projection by choosing a virtual plane that best matches the contact locations. Sato~\cite{Sato2011} constructed the support area by heuristically predicting the \ac{COM} trajectory, and  Bellicoso~\cite{BellicosoJenelten2017} projected the footholds vertically on the virtual plane. These methods produce a support area independent of the \ac{COM} but introduce a projection error that grows with prediction uncertainty or terrain nonlinearity, respectively. On the other hand, the \ac{GIAC} model does not depend on the definition of a virtual plane and does not produce a projection error.

\subsection{Relation to Wrench Models} \label{a:comparison_cwc}
If each contact force satisfies the no-slip condition~\eqref{a:no_slip_condition}, then the \ac{CW} $\vec w^\text{c}$ lies in a six dimensional \ac{CWC}, encompassing projections of all friction cones via the mapping defined by $\vec w^\text{c}$ in~\eqref{eq:models_cd}. By linearizing the friction cones, the \ac{CWC} becomes a convex polyhedral cone, which can be written as~\cite{CaronZMP}
\begin{equation} \label{a:models:cwc}
    \mathcal{C}^\text{CWC} = \left\{\sum\nolimits_{i,j} \lambda_{ij} \Mx{ \vec e_{ij} \\ \vec p_i \times \vec e_{ij}} ~\mid \lambda_{ij} \geq 0 \right\}.
\end{equation}
This form is called the \emph{span} or \emph{ray} representation of the cone. The unit vector $\vec e_{ij}$ denotes the $j$th facet of the $i$th friction pyramid. According to polyhedral cone theory~\cite{double_description}, it is possible to eliminate the slack variables $\lambda_{ij}$ and write~\eqref{a:models:cwc} in the so called \emph{face} or \emph{half-space} representation, e.g. by leveraging numerical tools such as the \ac{DDM}
\begin{equation}
    \mathcal{C}^\text{CWC} = 
    \underbrace{\{ \mat V \vec z \mid \vec z \geq 0 \}}_{\text{span}(\mat V)} \overset{\text{DDM}}{\iff}
     \mathcal{C}^\text{CWC} = 
    \underbrace{\{ \vec x \mid \mat U \vec x \leq 0 \}}_{\text{face}(\mat U)}.
\end{equation}
The face representation is preferred in \ac{TO} methods as it constitutes a minimal representation of the cone. 

Due to the inherent relation $\vec w^\text{gi} + \vec w^\text{c} = \vec 0$~\eqref{eq:models_cd}, both gravito-inertia and contact wrench belong to the same cone, i.e.,
\begin{equation}
    \vec w^\text{c} \in \mathcal{C}^\text{CWC}  \iff \vec w^\text{gi} \in \mathcal{C}^\text{CWC}.
\end{equation}

Iff
\begin{equation} \label{eq:models:cwc_stability}
    \vec w^\text{gi} \in \mathcal{C}^\text{CWC},
\end{equation}
then the resulting motion is \emph{weak-contact stable}~\cite{Caron2015}.
\begin{definition}[Weak-contact stability~\cite{contact_stability}] A motion is said to be weak-contact stable iff there exists a set of contact forces $\{\vec f_i\}_{\highlight{1,\ldots, N}}$ s.t. the \ac{EOM}~\eqref{eq:models_cd} are satisfied and the contact forces are bounded by its Coulomb friction cones~\eqref{eq:models_cd_friction_cone}.
\end{definition}

Weak contact stability has led to a new branch of research in the \ac{TO}-community, namely the \emph{wrench models}. Related methods follow a five-staged \ac{MMO} structure: 1) foothold selection; 2) geometric construction of the \ac{CWC}; 3) conversion of the geometric cone to inequality constraints; 4) Motion optimization subject to the stability criterion~\eqref{eq:models:cwc_stability}; And 5) computation of joint torques that realize stable contact forces. Step 4) can be seen as the bottleneck because the conversion is of pure algorithmic nature, and footholds need to be provided as known parameters.

A similar (perhaps less general) conclusion holds for the \ac{GIAC} model. From prop.~\ref{a:dyn_prop1} and~\ref{a:dyn_prop2} we can deduce that:
\begin{prop}[Generalization of \ac{GIAC} stability] \label{prop:models:generalization_cone}
If 
\begin{itemize}
    \item the \ac{GIAC} constraints \eqref{eq:models:m1} to \eqref{eq:models:mN} are satisfied,
    \item the rate of change of the angular momentum has a negligible effect on the  contact  forces,
    \item the torso is over-hanging, and
    \item the ground is flat (and tilted), or can be segmented into planes perpendicular to gravity, 
\end{itemize}
then the motion of the robot is weak-contact stable.
\end{prop}
We note that prop.~\ref{prop:models:generalization_cone} guarantees weak-contact stability of our model under assumptions~\ref{ss:model_assumptions}. 

Apart from being described by \ac{SRBD}, wrench models assume only polyhedral Coulomb friction. Even if assumption~\ref{ss:model_assumptions} is satisfied, \ac{CWC} constraints may appear less conservative than \ac{GIAC} constraints. This can be seen by prop.~\ref{prop:models:zmp_linear} or \ref{a:dyn_prop1} and~\ref{a:dyn_prop2}, which hold with ``if'' but not with ``iff''. The additional assumption of flat ground leads to a special case, for which the propositions hold with ``iff''.


\section{} \label{a:gauss_newton}
We consider a constrained \ac{NLP}~\eqref{eq:method:opt} and impose the limitation that the objective functions $f_i:\mathbb{R}^n\to\mathbb{R}$ can be factorized with some vector-valued functions $\vec g_i:\mathbb{R}^n\to \mathbb{R}^m$,
\begin{equation}
    f_i(\vec x) = 0.5\vec g_i(\vec x)^T \mat W_i \vec g_i(\vec x), \quad \mat W_i = \mat W_i^T \succ \vec 0.
\end{equation}
By approximating $\vec g_i$ by it's first order Tailor expansion about $\vec x_0$, i.e., $\vec g_i(\vec x) \approx \vec g_i(\vec x_0) + \nabla \vec g_i^T(\vec x_0) (\vec x - \vec x_0)$,
we can compute the Jacobian and approximate Hessian of the objective as
\begin{align}
    \mat J_i(\vec x_0) &= \nabla \vec g_i(\vec x_0)^T \mat W_i \vec g_i(\vec x_0) \\
    \mat H_i(\vec x_0) &= \nabla \vec g_i(\vec x_0)^T \mat W_i \nabla \vec g_i(\vec x_0). \label{eq:a:H}
\end{align}
It is easy to verify that the Hessian matrix~\eqref{eq:a:H} is always positive semi-definite. This technique is generally known as the \emph{Gaussian-Newton} approximation.

%% file: chapters/biography.tex
\begin{IEEEbiography}[{\includegraphics[width=1in,height=1.25in,clip,keepaspectratio]{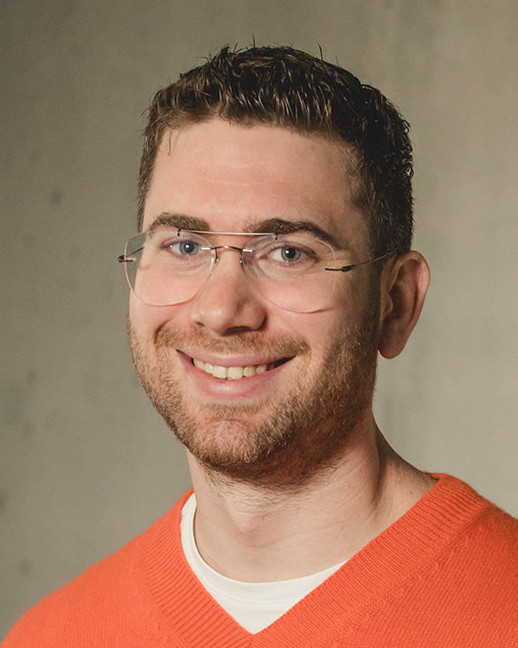}}]{Fabian Jenelten}
is a Ph.D. student at the Robotic Systems Lab, ETH Zurich, under the supervision of Prof. M. Hutter. He is currently working toward a Ph.D. degree in the field of trajectory optimization and learning-based control. He received his B.Sc. and M.Sc. in Mechanical Engineering from ETH Zurich, Switzerland, in 2015 and 2018.
His research interests include the combination of model- and learning-based control for legged locomotion.\vspace{-4mm}
\end{IEEEbiography}%
\begin{IEEEbiography}[{\includegraphics[width=1in,height=1.25in,clip,keepaspectratio]{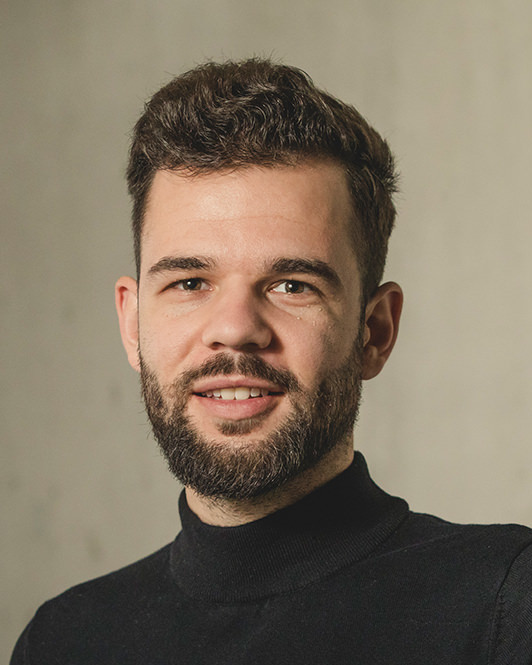}}]{Ruben Grandia}
received his B.Sc. in Aerospace Engineering from TU Delft, the Netherlands, in 2014, and his M.Sc. degree in Robotics, Systems, and Control from ETH Zurich, Switzerland, in 2017.
He is currently working toward a Ph.D. degree in the field of motion planning and control for legged robots at the Robotic Systems Lab at ETH Zurich, under the supervision of Prof. M. Hutter. 
His research interests include nonlinear optimal control and its application to dynamic mobile robots.\vspace{-4mm}
\end{IEEEbiography}%
\begin{IEEEbiography}[{\includegraphics[width=1in,height=1.25in,clip,keepaspectratio]{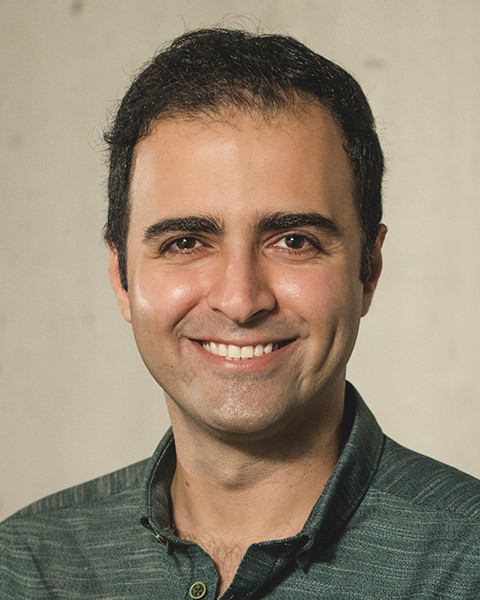}}]{Farbod Farshidian}
is a Senior Scientist at Robotic System Lab, ETH Zurich. 
He received his M.Sc. in electrical engineering from the University of Tehran in 2012 and his PhD from ETH Zurich in 2017, focusing on the planning and control of legged robots. His research interests are in mobile robots' motion planning and control, aiming to develop algorithms and techniques to endow these platforms to operate autonomously in real-world applications. Farbod is part of the NCCRs Robotics and Digital Fabrication. \vspace{-4mm}
\end{IEEEbiography}%
\begin{IEEEbiography}[{\includegraphics[width=1in,height=1.25in,clip,keepaspectratio]{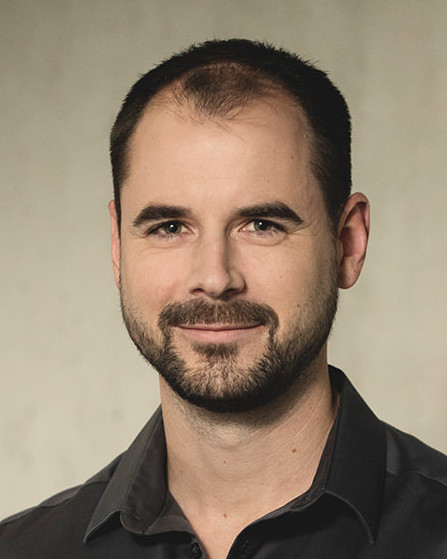}}]{Marco Hutter}
is Associate Professor for Robotic
Systems at ETH Zurich. He received his M.Sc. and PhD from ETH Zurich in 2009 and 2013 in the field of design, actuation and control of legged robots. His research interests are in the development of novel machines and actuation concepts together with the underlying control, planning, and machine learning algorithms for locomotion and manipulation. Marco is recipient of an ERC Starting Grant, PI of the NCCRs robotics and digital fabrication, PI in various EU projects and international challenges, an co-founder of several ETH Startups such as ANYbotics AG.
\end{IEEEbiography}

\vfill



